\documentclass[11pt]{article}

\usepackage[margin=1in]{geometry}
\usepackage{amsthm,amsmath,amsfonts,amssymb}
\usepackage[authoryear]{natbib}
\usepackage[colorlinks,citecolor=blue,urlcolor=blue]{hyperref}
\usepackage{graphicx}
\usepackage{microtype}
\usepackage{makecell}
\usepackage{booktabs}
\usepackage{multirow}
\usepackage{algorithm}
\usepackage{algorithmic}

\newcommand{\E}{\mathbb{E}}
\newcommand{\R}{\mathbb{R}}
\newcommand{\KL}{D_{\mathrm{KL}}}

\DeclareMathOperator*{\argmin}{arg\,min}
\newcommand{\cP}{\mathcal{P}}
\newcommand{\cPcal}{\cP_{p}}

\theoremstyle{plain}
\newtheorem{theorem}{Theorem}[section]
\newtheorem{proposition}[theorem]{Proposition}
\newtheorem{lemma}[theorem]{Lemma}
\newtheorem{corollary}[theorem]{Corollary}
\theoremstyle{definition}
\newtheorem{definition}[theorem]{Definition}
\newtheorem{assumption}[theorem]{Assumption}
\newtheorem{remark}[theorem]{Remark}

\begin{document}

\title{EMS Coreset: An Efficient Expectation-Maximization Algorithm for Sinkhorn Coreset}
\author{%
Haoyun Yin\thanks{Equal contribution.}
\quad Chuanhui Liu\footnotemark[1]
\quad Xiao Wang\\[0.5em]
\small Department of Statistics, Purdue University\\
\small West Lafayette, IN 47906, USA\\
\small \texttt{yin164@purdue.edu}, \texttt{liu2306@purdue.edu}, \texttt{wangxiao@purdue.edu}}
\date{}
\maketitle

\begin{abstract}
Coresets distill large datasets into small, representative subsets for efficient downstream learning. Yet Optimal Transport (OT)–based selection typically requires intensive computation of transport plans, limiting scalability. We introduce a scalable Sinkhorn coreset method that permits closed-form updates of the entropically regularized OT coupling by allowing non-uniform coreset weights. This produces centroids that generalize k-means via soft assignments. We establish asymptotic consistency of the selected measure and Lipschitz stability to data perturbations, providing accuracy and robustness guarantees. Across synthetic and real-world benchmarks, the proposed method achieves competitive or improved approximation quality while substantially reducing runtime compared to Wasserstein- and standard Sinkhorn-based coreset selection, especially at large scale.
\end{abstract}


\section{Introduction}
\label{sec:intro}
\begin{figure}[t]
  \centering
  \includegraphics[width=1.01\linewidth]{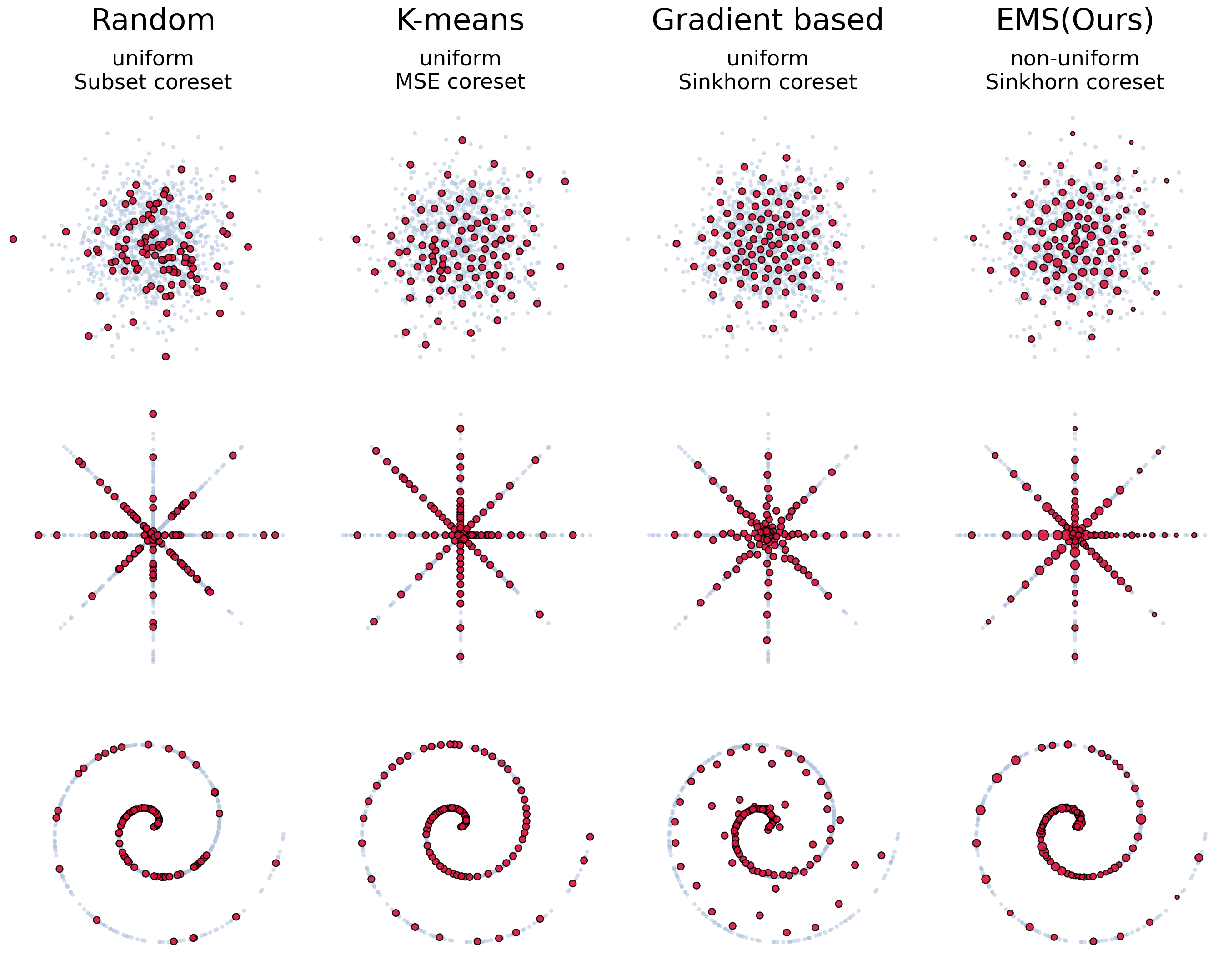}
  \caption{Comparison on 2D manifold datasets. All optimal-transport coresets capture the global geometry. EMS (Ours) spreads smoothly and preserves low-density tails, while gradient-based WCSL can be over-smoothed.}
  \label{fig:overview-2d}
\end{figure}
Coreset methods are dataset distillation techniques designed to address the escalating computational and memory demands of large-scale learning by identifying small yet representative subsets of data; see \citet{lei2023comprehensive, Yu2024review,moser2025coreset}
for comprehensive reviews. By preserving the data geometry with a substantially smaller sample size, downstream tasks achieve performance comparable to the original dataset across various clustering, regression, and inference tasks  \citep{har2004coresets,campbell2018bayesian, mirzasoleiman2020coresets, chhaya2022coresets, xia2024refined}. Coreset methods are becoming increasingly important in frontier applications in AI, including fine-tuning of billion-scale foundation models \citep{san2024in2core, maekawa2024dilm}, improving efficiency in image classification and object detection \citep{borsos2020coresets, cazenavette2022dataset}, and scaling modern benchmarks across multi-modal and federated learning \citep{wu2023vision,guo2024fedcore}.

Recent work leverages Wasserstein \citep{villani2008optimal} and entropic regularized Wasserstein (Sinkhorn) distance \citep{cuturi2013sinkhorn, altschuler2017near} in Optimal Transport (OT) theory to construct coresets \citep{claici2018wasserstein,izzo2021dimensionality,Liu2023DD, xiong2024fair, yin2025wasserstein}. Unlike the Feldman–Langberg sub-sampling paradigm, which selects representatives from the observed data under a query-space and relative-error setup, the OT-based coresets preserve geometric structure by minimizing transport cost to the full dataset, which connects to classical notions of barycenter \citep{agueh2011barycenters,cuturi2014fast}, representative points \citep{flury1990principal, fang1993number, MakSimon2018SP} and also aligns with moment-matching herding \citep{welling2009herding}. In addition, OT-based coreset methods are no longer restricted to uniform sample weights \citep{braverman2022power}, thereby enabling dataset compression efficiency through adaptive weighting schemes. We defer a more detailed comparison to the related work section.

However, existing OT-based coreset methods are computationally expensive. A major bottleneck lies in evaluating the similarity metric—namely, the Sinkhorn distance between a candidate coreset and the full dataset—which requires solving an OT problem. For example, computing the OT plan in Wasserstein distances scales cubically with dataset size \citep{villani2008optimal}, and it remains nearly quadratic via Sinkhorn distances \citep{cuturi2013sinkhorn, altschuler2017near, genevay2019sample}. Moreover, constructing an OT-based coreset necessitates repeated evaluation of this metric until convergence. By contrast, classical coreset selection theory in geometric optimization often provides $(1+\varepsilon)$-approximations with linear or near-linear construction time \citep{agarwal2005geometric,har2004coresets}. As a result, OT-based coreset methods suffer limitations from computational complexity, making them impractical for large-scale datasets.

This gap motivates our work to develop an efficient Sinkhorn coreset construction method. Our proposed approach, called \textbf{EMS}, is an EM-style algorithm that replaces explicit OT computation with efficient iterative updates. Unlike classical OT-based coreset that explicitly solves a discrete OT plan, the E-step computes a variational OT plan by taking expectations under a Gibbs distribution in function space, minimizing a variational upper bound on the Sinkhorn distance. In the M-step, the representative locations and weights are updated to maximize the expected assignment likelihood, ensuring a coherent and fully variational update of the discrete measure. Figure~\ref{fig:overview-2d} highlights the smooth coverage EMS achieves on 2D manifold benchmarks. Our contributions are three-fold:

\begin{enumerate}
    \item We provide a comprehensive analysis of weighted Sinkhorn coresets, covering standard results on geometry-preserving property, such as finite-atom approximation bounds and asymptotic consistency, to algorithm-specific guarantees, including approximation error, convergence, and stability theorems, offering a solid foundation for our proposed method. To better highlight the importance of non-uniform weights in the setup, we compared our results with the vanilla k-means method.

    \item We provide a detailed derivation of our proposed EMS algorithms. We reinterpret the weighted Sinkhorn coreset optimization problem from a variational perspective, formulating it as a soft assignment of points to coreset atoms, which naturally leads to the Gibbs-form solution (Eq.\ref{eq:estep}) and an equivalent reduced objective that avoids repeated matrix scaling with fixed weights. This approach draws a clear connection to variational inference in function space, treating the coreset as a tractable surrogate that approximates the original measure under entropic regularization.

    \item Comprehensive evaluations of our method are conducted in various data structures and also in downstream tasks. Our method improves upon classical baselines in distributional accuracy and stability, and achieves performance comparable to OT-based coresets obtained via explicit gradient-based optimization while being over $100\times$ more efficient, consistent with our theoretical guarantees. For practice, we also propose a mini-batch variant that further improves practical scalability.
\end{enumerate}

\section{Related Work}
\paragraph*{Sensitivity-based Coreset Selection}
Classical coreset methods rely on sensitivity-based sampling from existing data points with respect to a specified query space \cite{feldman2011unified}. In the case of likelihood-based or Bayesian inference, Bayesian coresets approximate the full-data log-likelihood by a sparse weighted subset of data points, yielding an efficient surrogate posterior for inference  \cite{huggins2016coresets, campbell2018bayesian, manousakas2020bayesian, zhang2021bayesian}. In comparison, OT-based coresets are constructed to approximate the discrete empirical data distribution in Wasserstein space. For example, Wasserstein or Sinkhorn coresets \citep{claici2018wasserstein, xiong2024fair, yin2025wasserstein} are distribution-model free in the sense that they do not assume a parametric likelihood and do not approximate a log-likelihood or posterior. That said, their objective are also shown to be equivalent via an integral probability metric (IPM) \citep{claici2018wasserstein} under an
$\varepsilon$-uniform approximation capacity framework \citep{yin2025wasserstein}. Such equivalence follows from a standard application of Kantorovich–Rubinstein duality \citep{villani2008optimal, arjovsky2017wasserstein}. Consequently, OT-based coresets are not limited to finite data points or pointwise-defined target objective and can admit $(1+\epsilon)$-approximation guarantees over a function class.


\paragraph*{Barycenters and Representative Points (RPs)}
Both barycenters and RPs also aim to compress probability measures, but offers conceptually  complementary views. Wasserstein barycenters minimize the average transport cost to \textit{multiple input measures}, focusing on optimization problem itself\citep{agueh2011barycenters}, whereas representative points emphasis on asymptotic distributional convergence of its sparse approximation under difference statistical discrepancy (e.g., energy distance) \citep{flury1990principal,MakSimon2018SP}. OT-based coresets can be viewed as a constrained case of barycenters of single discrete measure or RPs using a OT distance metric. Connecting these dots, we not only advance OT-based coreset theory by showing that Sinkhorn coresets avoid the systematic density approximation bias inherent to uniformly weighted centroids under squared loss, which characterizes classical $k$-means quantization \citep{graf2007foundations}. Moreover, our proposed EMS algorithm admits smooth, closed-form gradient updates analogous to free-support Wasserstein barycenters \citep{cuturi2014fast}, while providing a more rigorous variational and optimization-theoretic derivation.

\paragraph*{Variational, Schr\"odinger, and EM views of entropic OT.}
Entropic OT can be derived as a KL‑penalized projection or as iterative Bregman projections \citep{benamou2015iterative}, and it admits a Schr\"odinger bridge interpretation that links transport plans to Gibbs measures \citep{leonard2013survey}. This variational structure underpins differentiability and stability results for Sinkhorn couplings and costs \citep{ghosal2022stability}, and also yields useful statistical interpretations (e.g., as maximum‑likelihood deconvolution) \citep{rigollet2018entropic}. Applying these insights to the OT-based coreset problem, our derivation leads to an EM-style formulation for OT-based coreset construction. Our E-step directly employs the Gibbs form of the entropically regularized coupling, yielding assignments that are stable and largely insensitive to the choice of Sinkhorn regularization hyperparameter $\lambda$. The M-step then admits closed-form updates for both representative locations and \emph{non-uniform} weights. Unlike gradient-based Sinkhorn coreset methods\citep{yin2025wasserstein}, our EMS run more efficiently in the small-$\lambda$ regime and recovers a $k$-means-like hard-assignment procedure as $\lambda\to0$ \citep{mi2018variational}.

\section{Sinkhorn Coreset}\label{sec:sinkhorn-coreset}
In this section, we formalize the Sinkhorn coreset problem and demonstrate its geometry-preserving properties.

\subsection{Weighted Coresets via Sinkhorn Loss}
Our goal is to compress a given probability measure $\mu \in \cPcal(\R^d)$ with finite $p$-th moment, primarily in the discrete empirical setting relevant to coreset construction, which also natural extends to absolutely continuous measures.
\begin{equation}
\mu = \tfrac{1}{n} \textstyle \sum_{i=1}^n\delta_{x_i} ~ \text{or}  ~ d\mu(x) = p(x)\,dx.
\end{equation}
OT-based coreset methods seek to approximate $\mu$ via a discrete measure $\nu_k^*:=\sum_{j=1}^{k}\omega_j\delta_{y_j}$, defined by its locations $Y:=\{y_j\}_{j=1}^k$ and weights $W=\{\omega_j\}_{j=1}^k$ under OT-based metric $d(\cdot,\cdot)$ such that, for $k \ll n$,
\begin{equation}
\nu_k^*=\argmin_{\nu_k\in\mathbb{G}_k}d(\mu,\nu_k),
\end{equation}
where $\mathbb{G}_k = \{
\textstyle\sum_{j=1}^{k}\omega_j\delta_{y_j} | \textstyle \sum_{j}\omega_j=1, \omega_j\ge 0, y_j\in\R^d\}$ is the set of all discrete measures with at most $k$ atoms.

A canonical example $d(\cdot,\cdot)$ is the ($p$-powered) Wasserstein-$p$ distance $W_p^p$, which is given by
\begin{equation}\label{eq:defWS}
  W_p^p(\mu,\nu)
  := \inf_{\pi\in\Pi(\mu,\nu)} \int_{R^d\times R^d} \|x-y\|^p \, d\pi(x,y)
\end{equation}
where the cost $c(x,y)$ is $p$-powered $\mathcal{\ell}^p$-norm distance and $\Pi(\mu,\nu)$ denotes the set of couplings $\pi$ with marginals being $\mu$ and $\nu$.

When $\nu$ is discrete, any coupling $\pi\in\Pi(\mu,\nu)$ can be represented by a \emph{transport plan}
$T=(T_1,\ldots,T_k):\R^d\to\Delta_k$, which induces a semi-discrete coupling
$\pi_T$. The corresponding transport cost is its expected cost under $\pi_T$,
\[
\int \|x-y\|^p\,d\pi_T(x,y)
=\sum_{j=1}^k \int_{\R^d} \|x-y_j\|^p\,T_j(x)\,d\mu(x).
\]
Define the cost vector $C=(C_1,\ldots,C_k)$ with $C_j(x):=\|x-y_j\|^p,j\in[k]$, Eq.\ref{eq:defWS} admits the equivalent form
\begin{equation}\label{eq:def-WS}
W_p^p(\mu,\nu)
=\inf_{T\in\mathbb{T}_k}
\sum_{j=1}^k\int_{\R^d}C_j(x)\,T_j(x)\,d\mu(x):=\langle C,T\rangle_\mu,
\end{equation}
where $\mathbb{T}_k$ denotes the function space of admissible plans
\[
\mathbb{T}_k
:=\Big\{T:\R^d\to\Delta_k \ \big|\
\int_{\R^d} T_j(x)\,d\mu(x)=\omega_j,\ \forall j\in[k]\Big\}.
\]

The Sinkhorn regularization on $W_p^p(\mu,\nu)$ \citep{cuturi2013sinkhorn} controls the Kullback–Leibler divergence between $\pi$ with a trivial and independent coupling $\mu\otimes\nu$, which leads to the \textit{Sinkhorn distance}:
\begin{equation}\label{eq:def-sinkhorn}
W_{\lambda,p}^p(\mu,\nu):=\inf_{T\in\mathbb{T}_k}\!\left\{\langle C,T \rangle_\mu+\lambda\,\KL\!\big(\pi_T\,\Vert\,\mu\otimes\nu\big)\right\},
\end{equation}
where  $\lambda >0$ is the regularization strength. For any $\lambda>0$, the regularization ensures convexity in $\pi_T$, allowing efficient computation of $T$ via Sinkhorn iterations \citep{sinkhorn1967concerning}. Though biased, the resulting optimal $\pi$ (or $T$) realizes a discrete-time Schrödinger bridge \citep{leonard2013survey}. \cite{genevay2019sample} generalize Eq.\ref{eq:def-sinkhorn} to a proper distance, interpreting it as an interpolation between Wasserstein and Maximum Mean Discrepancy (MMD) distances.

A Sinkhorn Coreset aims to approximate $\mu$ under this metric
\begin{equation}\label{eq:sinkhorn-coreset-T}
\begin{aligned}
\nu_k^*
&:=\argmin_{\nu_k \in\mathbb{G}_k}W_{\lambda,p}^p(\mu,\nu_k) \\
&:=\argmin_{\nu_k(Y,W)\in\mathbb{G}_k,\ T\in \mathbb{T}_k}
\mathcal{F}_\lambda(Y,W,T),
\end{aligned}
\end{equation}
where we define the Sinkhorn loss $\mathcal{F}_\lambda$ as
\begin{equation}\label{eq:sinkhornloss}
\mathcal{F}_\lambda(Y,W,T):=\langle C,T \rangle_\mu + \lambda \KL\!\big(\pi_T\,\Vert\,\mu\otimes\nu\big).
\end{equation}
Notably, coreset construction via objective $\mathcal{F}_\lambda$ shifts the focus from merely obtaining an optimal coupling $\pi_T$ (or equivalently $T$) to a distributional geometry-preserving goal as $\nu_k^*$ summarizes $\mu$ with $k<<n$, as formalized below.

\subsection{Geometry-Preserving Guarantees}
\begin{theorem}[Sinkhorn Bound]\label{thm:finite-k-bound}
If the target \(\mu\) has a bounded density, compact support, and finite \((p+\delta)\)-moment for some \(\delta>0\)  and \(\lambda>0\), for any \(k\ge 1\), then for any global minimizers $\nu^*_k \in\mathbb{G}_k$ of Eq.\ref{eq:sinkhorn-coreset-T},
\begin{equation}\label{eq:finite-m-bound}
  W_{\lambda,p}^p\!\big(\mu,\nu^*_k\big)
  \ \le\  C k^{-1/d} + C' k^{-2/d}\,
\end{equation}
where constants \(C,C'\) depend only on \(\mu,p,d\).
\end{theorem}
Proof can be found in Appendix~\ref{apx:proof-finite-k-bound}.
To obtain this upper bound, we first adapt the classical quantization bound for the Wasserstein distance via deterministic Voronoi assignments. Next, we bound the KL divergence with a tight bound using Taylor expansion, thereby concluding the proof.
\begin{corollary}[Geometry Persevering Asymptotic Consistency]\label{thm:consistency}
For any global minimizers of Eq.\ref{eq:sinkhorn-coreset-T}, we have
\begin{equation}\label{eq:anneal-condition}
\lim_{k\to\infty}W_p\!\big(\mu,\nu^{*}_{k}\big)\ =\ 0
\end{equation}
\end{corollary}
See Appendix~\ref{apx:proof-consistency}. Theorem~\ref{thm:finite-k-bound} provides a finite-sample guarantee for Sinkhorn coresets, which matches the classical quantization rate. This indicates that a coreset of size $k$ approximates $\mu$ efficiently while preserving its geometric structure. Corollary~\ref{thm:consistency} further establishes asymptotic consistency: as $k \to \infty$, the coreset distribution converges to the target in Wasserstein distance, ensuring that the underlying geometry of $\mu$ is faithfully recovered in the limit.

\subsection{On Non-uniform Weights}
\label{subsec:kmeans-inconsistent}

We contrast our weighted Sinkhorn construction with classical $k$-means under squared Euclidean loss, whose centers are uniformly weighted in the large-$k$ regime common in high-resolution quantization theory.
\begin{lemma}[Uniform-weight $k$-means are not consistent \citep{gersho1979asymptotically}]
\label{cor:kmeans-inconsistent}
Let $\nu_k^{\mathrm{KM}}:=\frac{1}{k}\sum_{j=1}^k \delta_{y_j}$ be the uniform-weight empirical measure supported on $Y_k$. Denote weak convergence as $\Rightarrow$. Under the high-resolution density assumption established in Appendix~\ref{AP:highres}, we have
\begin{equation}
\begin{aligned}
\nu_k^{\mathrm{KM}}
\Rightarrow\ g^\star(y)\,dy,
\quad\text{with }
g^\star(y)
=\frac{p(y)^{\frac{d}{d+2}}}{\int p^{\frac{d}{d+2}}}\,,
\end{aligned}
\end{equation}
\end{lemma}
As a result, $\nu_k^{\mathrm{KM}}$ fails to converge to $\mu$ for any $p\ge1$ in general, unless the target distribution has uniform density on its support. Formal derivations, regularity assumptions, and the limiting density are provided in Appendix~\ref{AP:highres}.

\begin{proposition}[Weighted correction recovers $\mu$]
\label{prop:weights-fix}
For $Y_k$ defined in Appendix \ref{AP:highres}. Define weights
\begin{equation}
\omega_j \propto \frac{p(y_j)}{g^\star(y_j)},
\nu_k^{\,w} :=\sum_{j=1}^k \omega_j\,\delta_{y_j},
\sum_{j=1}^k \omega_j = 1.
\end{equation}
Then $\nu_k^{\,w}\Rightarrow \mu$ as $k\to\infty$.
\end{proposition}

Please refer to Appendix~\ref{apx:proof-weights-fix}. Proposition~\ref{prop:weights-fix} formalizes how \emph{non-uniform} weights remove this bias. Consequently, our Sinkhorn coreset formulation leverages this flexibility through adaptive weighting in $\nu_k$, in contrast to fixed-weight Sinkhorn coreset approaches, as seen in \citet{yin2025wasserstein}.

\section{An Efficient Algorithm for Sinkhorn Coresets}\label{gen_inst}
In this section, we proposed our EMS algorithm. First, we use the notion of the free energy functionals and its Fenchel identity to derive the closed form update of $T$. Here, recall that the transport plan $T:\mathbb{R}^d\to \Delta_k$ in Eq.\ref{eq:sinkhornloss} for $\mu$ maps each point $x\in \R^d$ to a $(k-1)$ simplex $\Delta_k$ and its admissible set $\mathbb{T}_k$ ensures that $T$ is a valid assignment for every $x$.
\subsection{Variational Inference in Function Space}

\begin{lemma}[Fenchel Identity of Free Energy Functional \citep{fenchel1953convex}]\label{lemma:Fenchel}
Let $(\mathbb{T}_k,\lVert\cdot\rVert)$ denote the Banach space of admissible transport plans
$T:\mathbb{R}^d \to \Delta_k$.
Let $P_0$ be a reference probability measure of $T$ on $\mathbb{T}_k$.
For a linear cost functional
$f(T)$,
define the \emph{free energy functional} $F_{f,P_0}$ by
\begin{equation}
F_{f,P_0} := \log \, \mathbb{E}_{T \sim P_0}\!\big[ e^{f(T)} \big].
\end{equation}
Then,  $F_{f,P_0}$ is convex in $f$ and admits the variational representation (Fenchel Identity):
\begin{equation}
F_{f,P_0} \;=\; \sup_{q \ll P_0} \Big\{ \mathbb{E}_{T\sim q}[f(T)] - \mathrm{KL}(q \,\|\, P_0) \Big\},
\end{equation}
where the supremum is over all probability measures $q$ on $\mathbb{T}_k$ that are absolutely continuous with respect to $P_0$.
In particular, the optimizer is the Gibbs measure
\begin{equation}
q^*(T) \;\propto\; e^{f(T)} P_0(T).
\end{equation}
\end{lemma}

The well-known proof is attached in Appendix~\ref{apx:proof-fenchel} for reference. Lemma \ref{lemma:Fenchel} offers a Gibbs’ variational perspective on the log-partition function defined over the space of transport plans: Minimizing the free energy functional $F_{f,P_0}$ can be interpreted as finding the optimal “trade-off” for a variational distribution $q$ over the admissible set of transport plans, balancing the maximization of the expected cost with proximity to the reference measure $P_0$.

In what follows, we demonstrate that the Sinkhorn loss, which shares this variational nature, can be recovered from $F_{f,P_0}$ for an appropriate choice of $f(T)$ and $P_0$, thereby reformulating Sinkhorn coreset optimization as an alternative variational optimization scheme.

\begin{corollary}[Sinkhorn Loss as Negative Free Energy Lower Bound]\label{cor:sinkhorn-lowbound}
Under Lemma \ref{lemma:Fenchel}, if $f(T) := -\tfrac{1}{\lambda}\langle C,T \rangle_\mu$ with $\lambda>0$ and $P_0$ is chosen so that
\begin{equation}
\KL(q^*\|P_0) \;\equiv\; \mathbb{E}_{T\sim q^*}\!\Big[\KL(\pi_T\|\mu\otimes \nu_k)\Big].
\end{equation}
Then, by Jensen’s inequality, we obtain
\begin{equation} \label{eq:inq}
    \begin{aligned}
        -\lambda F_{f,P_0}
&= \mathbb{E}_{T\sim q^*}[\langle C,T \rangle_\mu]
+ \lambda \,\mathbb{E}_{T\sim q^*}[\KL(\pi_T\|\mu\otimes\nu_k)]\Big.\\
&\;\ge\; \langle C,\bar{T} \rangle_\mu + \lambda\KL(\pi_{\bar{T}}\|\mu\otimes\nu_k)\\
&=\mathcal{F}_\lambda(Y,W,\bar{T})
\geq \inf_{T\in\mathbb{T}_k} \mathcal{F}_\lambda(Y,W,T),
    \end{aligned}
\end{equation}
where the optimal variational $\bar{T}:=\mathbb{E}_{T\sim q^*}[T]$.
\end{corollary}
Proof is deferred to Appendix~\ref{apx:proof-sinkhorn-lowbound}. In brief, the first inequality in in Eq.\ref{eq:inq} is tight if and only if the Gibbs measure $q^*$ collapses to a Dirac measure $\delta_{T}$, which occurs either when $P_0=\delta(T-T_0)$ or in the zero-temperature limit $\lambda\to 0$; The second equality further requires that the support point $T_0$ of such $P_0$ coincides with the global minimizer $T^*$.

In classical theory \citep{cuturi2013sinkhorn}, the optimal $T^*$ is obtained in closed form by solving Lagrange multipliers or KKT conditions. Corollary \ref{cor:sinkhorn-lowbound}, on the other hand, provides the probabilistic perspective of $T^*$, interpreted as the expectation under a Dirac measure. Such probabilistic viewpoint is particularly valuable when formulating the weighted Sinkhorn coreset optimization problem to a EM-style iterative algorithms.

Specifically, rather than explicity solving for $T^*$ that minimizes $\mathcal{F}_\lambda(Y,W,T)$ for a fixed $Y,W$, we instead minimize the variational objective $-\lambda F_{f,P_0}$ over probability measures $q$ on the function space of transport plans. The obtained expectation, i.e., the averaged transport plan $\bar{T}=\mathbb{E}_{q^*}[T]$, serves as an effective surrogate for optimization. The formal derivation is provided below.

\subsection{An EM-algorithm for Sinkhorn (EMS) Coreset}
\begin{theorem}[$E$-step via mean-field $P_0$]\label{thm:estep-factorizes}
A sufficient condition for $\KL(q^*\|P_0)=\mathbb{E}_{T\sim q^*}\!\Big[\KL(\pi_T\|\mu\otimes \nu_k)\Big]$ is that $P_0$ factorizes over $x\in \R^d$ with weights $\omega=(\omega_1,\dots,\omega_k)$. Under such condition,  $q^*$ then also "factorizes" over $x$, i.e.,
  $$q^*(T) = \prod_{x\in \R^d} q^*_x(T(x)).$$ Consequently, the expected transport plan $\bar{T}:=\mathbb{E}_{T\sim q^*}[T]$ corresponds to the vector-valued softmax
\begin{equation}\label{eq:estep}
\begin{aligned}
\bar{T}_j(x)=\frac{\omega_j \, \exp\big(-C_j(x)/\lambda\big)}
{\sum_{l=1}^k \omega_l \, \exp\big(-C_l(x)/\lambda\big)},
\quad j=[k].
\end{aligned}
\end{equation}

\end{theorem}

Please refer to Appendix~\ref{apx:proof-estep} for a detailed derivation. At a high level, although the resulting $\bar{T}$ may not correspond to the optimal $T^*$ in terms of minimizing the Sinkhorn distance \citet{cuturi2013sinkhorn}, it admits a closed-form expression that can be computed efficiently, under appropriate choices of reference measure $P_0$. This flexibility are similar to the E-step in EM algorithms, which replaces exact latent-variable optimization with tractable expected assignments without altering the underlying objective.

\begin{definition}[\textbf{M-step}]\label{def:M-step}
Let $\bar{T}^{(t+1)}$ be the transport plan obtained from the E-step. For $j = [k]$ abd at step $t+1$, M-step is defined as updating $Y,W$ for $\nu_k^{(t+1)}$ by
\begin{align}
y_j^{(t+1)} &= \frac{\int_{\mathbb{R}^d} x \, \bar{T}^{(t+1)}_j(x) \, d\mu(x)}{\int_{\mathbb{R}^d} \bar{T}^{(t+1)}_j(x) \, d\mu(x)}, \\
\omega_j^{(t+1)} &= \int_{\mathbb{R}^d} \bar{T}^{(t+1)}_j(x) \, d\mu(x).
\end{align}
\end{definition}
\begin{proposition}[M-step monotonicity and closed-form updates]\label{prop:descent-T}
The updates defined by Definition \ref{def:M-step} are the minimizer of the free-energy
\begin{equation}
\mathcal{F}_\lambda(Y, W, \bar{T}^{(t+1)})
\end{equation}
with respect to the cluster locations $Y = \{y_j\}_{j=1}^k$ and  weights in $W = \{\omega_j\}_{j=1}^k$.
In addition, the functional is non-increasing along the M-step:
\begin{equation}
\mathcal{F}_\lambda(Y^{(t+1)}, W^{(t+1)}, \bar{T}^{(t+1)})
\leq \mathcal{F}_\lambda(Y^{(t)}, W^{(t)}, \bar{T}^{(t+1)}),
\end{equation}
with equality if and only if $(Y^{(t)}, W^{(t)})$ satisfies the first-order optimality conditions for both blocks.
\end{proposition}
Proof is shown in Appendix~\ref{AP:Proof_Estep}. From a probabilistic viewpoint,  the M-step updates the cluster locations and weights to best fit the current soft assignments induced by the transport plan. The location updates correspond to weighted centers of mass, i.e., maximum
likelihood estimates of cluster means under the soft assignments, while the
weight updates maximize the expected log-likelihood of the mixture weights.

For clarity, a detailed EM-algorithm for Sinkhorn (EMS) coreset with discrete $\mu$ (empirical dataset) is shown below. The complexity analysis and accuracy of our algorithm are discussed in experiments.
\begin{algorithm}[tb]
\caption{EMS Coreset with Mini-batch Data}
\label{alg:pop-streaming-em}
\begin{algorithmic}[1]
\REQUIRE dataset $\{x_i\}_{i=1}^n$ (or i.i.d.\ draws from $p(x)$), $k\ll n$, $\lambda\!>\!0$, batch size $B$, tolerance $\varepsilon\!>\!0$
\STATE Initialise  $Y^{(0)}, W^{0}=\{y^{(0)}_j,\omega^{(0)}_j\}_{j=1}^k$
\REPEAT
    \STATE \textbf{reset} $N_j\gets 0$, $S_j\gets 0$ for $j=1,\dots,k$
    \STATE Partition $\{i\}_{i=1}^n$ into batches $\{\mathcal{B}_m\}_{m=1}^{n/B}$
    \FOR{$m=1$ to $n/B$}
        \FOR{$i\in\mathcal{B}_m$ and $j\in[k]$}
            \STATE $\ell_{ij}\gets \log \omega^{(t-1)}_j -  \|x_i-y^{(t-1)}_j\|^2/\lambda$
            \STATE $\bar{T}^{(t+1)}_{ij}\gets$ Softmax$(l_{ij})$ by row
        \ENDFOR
        \STATE $N_j\gets N_j+\sum_{i\in\mathcal{B}_m} \bar{T}^{(k)}_{ij}$
        \STATE $S_j\gets S_j+\sum_{i\in\mathcal{B}_m} \bar{T}^{(k)}_{ij}\,x_i$
    \ENDFOR
    \FOR{$j=1$ to $k$}
        \STATE $\omega^{(k)}_j \gets N_j\big/\sum_{\ell=1}^k N_\ell$\newline
        \STATE $y^{(t)}_j \gets
           \begin{cases}
             S_j/N_j,&\text{if } N_j>0,\\
             y_j^{(t-1)},&\text{otherwise.}
           \end{cases}$
    \ENDFOR
    \UNTIL{$\|Y^{(t+1)}-Y^{(t)}\|<\varepsilon$}
\STATE {\bfseries return} $Y^{(t+1)},W^{(t+1)}$
\end{algorithmic}
\end{algorithm}

\begin{figure}[t]
  \centering
  \includegraphics[width=.49\linewidth]{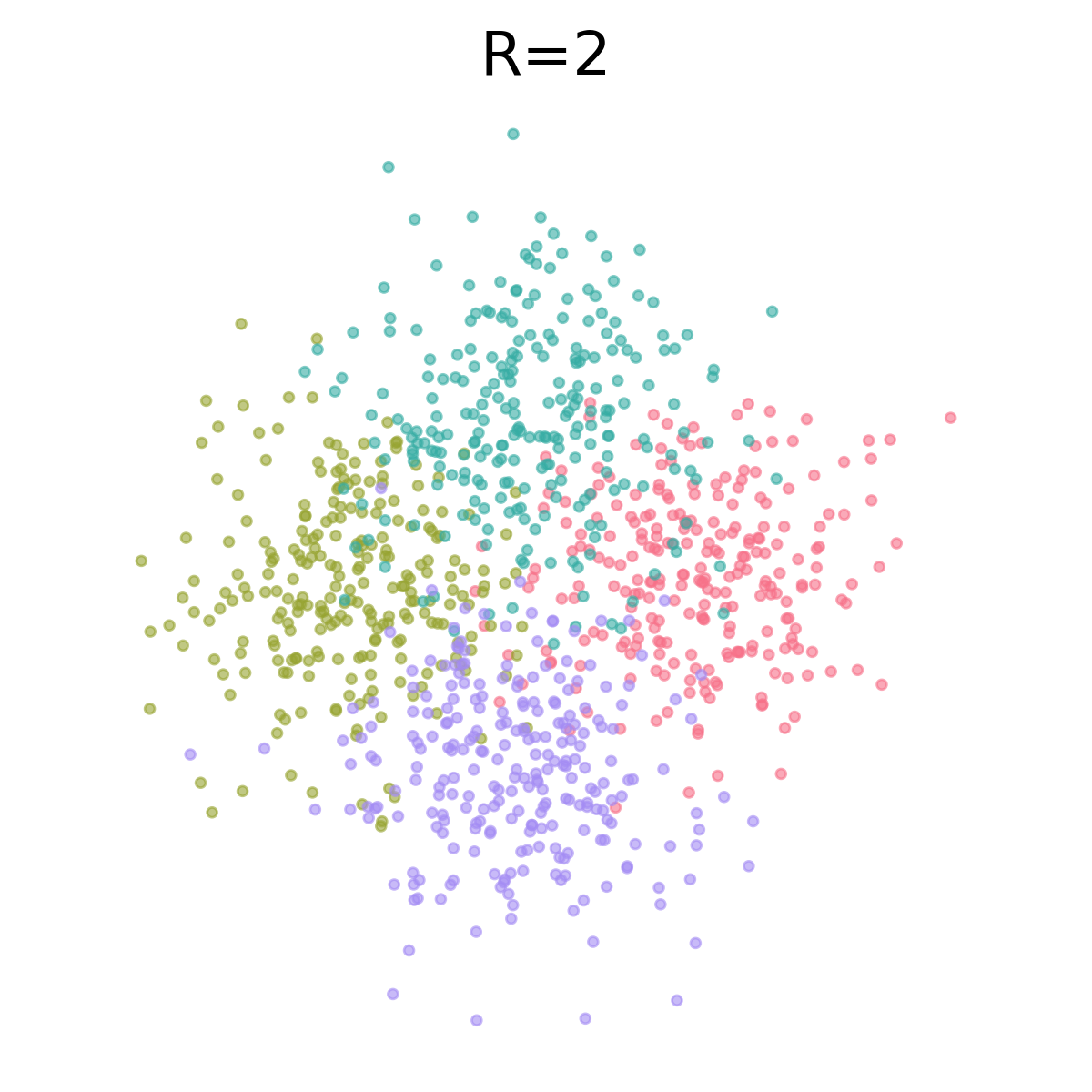}\hfill
  \includegraphics[width=.49\linewidth]{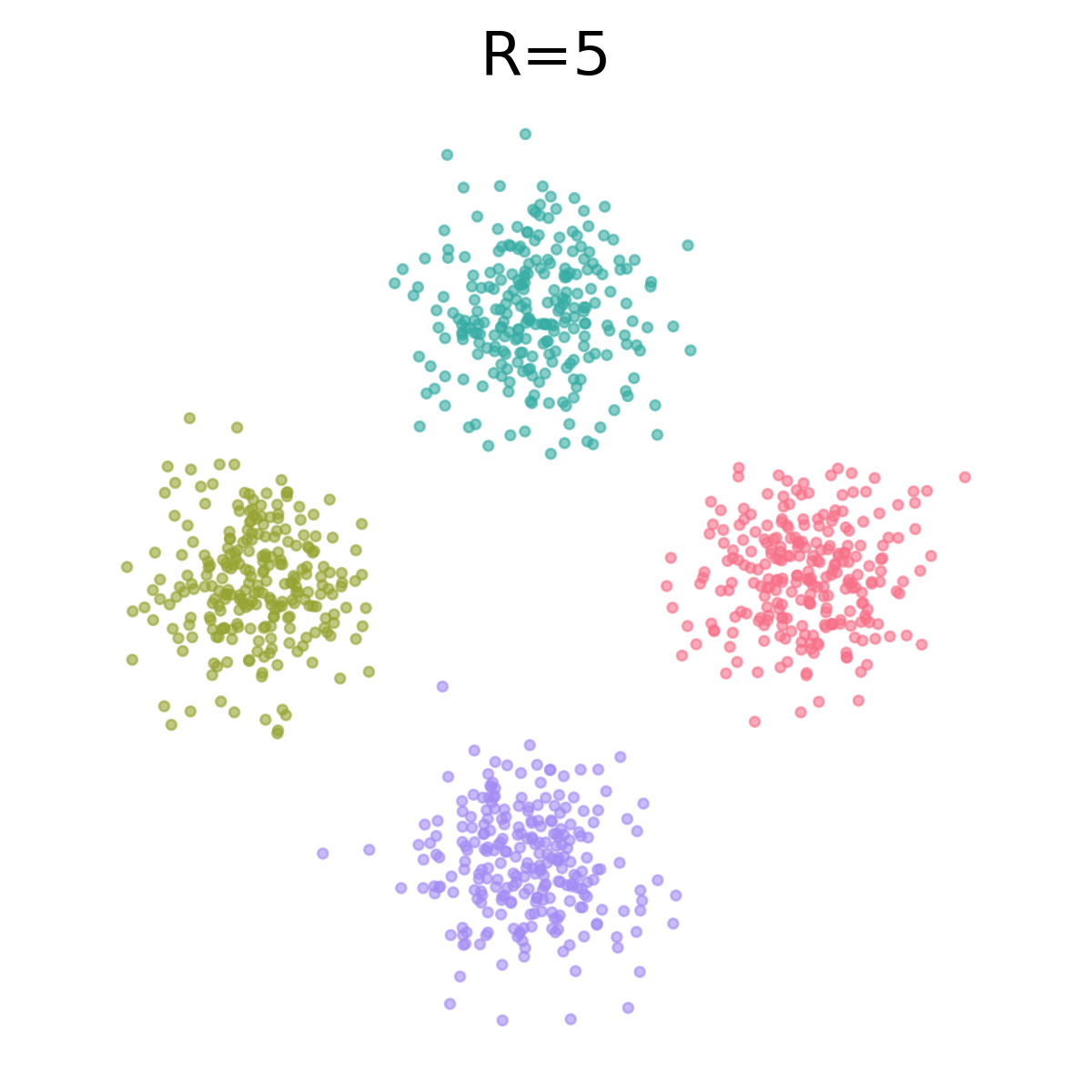}
  \caption{Synthetic Gaussian mixture datasets where each component has unit variance with its mean located at distance R from the origin.}
  \label{fig:gaussian}
\end{figure}

\begin{table}[t]
\centering
\caption{Distributional accuracy and total runtime (distance / time). Lower distance is better; bold indicates the best distance for each $k$. Statistics averaged over five runs. dist / time stands for distance / time.}
\label{tab:distance-runtime}
\small
\setlength{\tabcolsep}{6pt}
\resizebox{\textwidth}{!}{%
\begin{tabular}{lcccccccc}
\toprule
\multirow{3}{*}{Dataset}
& \multirow{3}{*}{Coreset size}
& {Random}
& {K-means}
& {WCSL}
& \multicolumn{4}{c}{EMS} \\
\cmidrule(lr){3-9}
& & & & $\lambda=10^{-2}$
& $\lambda=0$
& $\lambda=10^{-3}$
& $\lambda=10^{-2}$
& $\lambda=10^{-1}$ \\
&& \multicolumn{1}{c}{dist / time}
  & \multicolumn{1}{c}{dist / time}
  & \multicolumn{1}{c}{dist / time}
  & \multicolumn{1}{c}{dist / time}
  & \multicolumn{1}{c}{dist / time}
  & \multicolumn{1}{c}{dist / time}
  & \multicolumn{1}{c}{dist / time}\\
\midrule
\multirow{5}{*}{\textbf{Gaussian}}
& $k=10$  & 18.28 / - & 10.87 / 0.09 & \textbf{10.85} / 14.9 & 10.87 / \textbf{0.08} & 10.86 / 0.15 & 10.87 / 0.15 & 10.87 / 0.12 \\
& $k=20$  & 16.05 / - & 10.14 / 0.14 & \textbf{10.11} / 24.4 & 10.14 / \textbf{0.09} & 10.14 / 0.10 & 10.14 / 0.17 & 10.12 / 0.22 \\
& $k=50$  & 14.31 / - & 9.25 / 0.08 & \textbf{9.19} / 56.9 & 9.24 / \textbf{0.07} & 9.24 / 0.11 & 9.23 / 0.18 & 9.20 / 0.39 \\
& $k=100$ & 12.92 / - & 8.69 / \textbf{0.06} & \textbf{8.53} / 92.5 & 8.63 / \textbf{0.06} & 8.63 / 0.08 & 8.61 / 0.14 & 8.56 / 0.54 \\
& $k=200$ & 11.74 / - & 8.18 / 0.26 & \textbf{7.88} / 165.4 & 8.05 / \textbf{0.11} & 8.04 / 0.13 & 8.03 / 0.19 & 7.97 / 1.2 \\
\midrule
\multirow{5}{*}{\textbf{MNIST}}
& $k=10$  & 778.7 / - & 424.9 / 0.11 & 417.3 / 28.4 & 417.0 / 0.10 & 416.4 / 0.09 & \textbf{416.2} / \textbf{0.08} & 416.5 / \textbf{0.08} \\
& $k=20$  & 727.2 / - & 382.5 / 0.15 & 376.0 / 45.1 & 375.4 / 0.14 & \textbf{374.6} / \textbf{0.13} & 375.6 / \textbf{0.13} & \textbf{374.6} / 0.16 \\
& $k=50$  & 585.8 / - & 340.2 / 0.29 & 329.8 / 87.9 & 329.7 / 0.28 & \textbf{329.0} / \textbf{0.26} & 329.5 / 0.38 & 329.4 / 0.31 \\
& $k=100$ & 535.3 / - & 315.4 / 0.51 & \textbf{301.2} / 146.8 & 302.0 / 0.41 & 301.5 / 0.51 & 302.0 / \textbf{0.38} & 302.3 / 0.42 \\
& $k=200$ & 476.6 / - & 294.3 / 0.61 & \textbf{276.0} / 293.2 & 278.7 / \textbf{0.60} & 278.8 / 0.68 & 278.9 / 0.69 & 278.6 / 0.76 \\
\midrule
\multirow{5}{*}{\textbf{\makecell{Fashion\\MNIST}}}
& $k=10$  & 576.7 / - & 297.6 / 0.07 & 268.2 / 30.9 & 260.8 / 0.08 & \textbf{259.7} / 0.05 & 261.3 / \textbf{0.04} & 261.4 / \textbf{0.04} \\
& $k=20$  & 497.6 / - & 255.2 / \textbf{0.09} & 221.7 / 50.4 & 220.6 / 0.11 & \textbf{219.1} / 0.14 & 220.8 / 0.10 & 220.2 / 0.10 \\
& $k=50$  & 374.5 / - & 217.7 / \textbf{0.22} & 188.3 / 105.5 & \textbf{187.7} / 0.28 & 188.5 / 0.26 & 188.9 / 0.29 & 187.9 / 0.34 \\
& $k=100$ & 326.6 / - & 198.1 / 0.42 & \textbf{169.9} / 201.6 & 171.6 / \textbf{0.41} & 170.9 / 0.46 & 170.9 / 0.47 & 171.1 / 0.68 \\
& $k=200$ & 289.1 / - & 187.5 / 0.65 & \textbf{155.3} / 389.6 & 157.9 / 0.66 & 158.2 / \textbf{0.58} & 158.1 / 0.67 & 158.1 / 0.91 \\
\bottomrule
\end{tabular}%
}
\end{table}

\subsection{Robustness-Fidelity Trade-off}
\label{sec:stability}
Here, we show that the stability of our algorithm comes from Softmax-activated (Smoothed) $T$ updates. For coreset selection, it is important to ensure the transport plan does not change too wildly when data is perturbed.

\begin{proposition}[Bounded Gradient]\label{prop:bounded-gradient}
The gradient of the soft assignment from Eq.~\ref{eq:estep} satisfies
\begin{equation}
\label{eq:bounded-gradient}
\nabla_x \bar{T}_j(x)
=\frac{2}{\lambda}\,\bar{T}_j(x)\sum_{\ell=1}^{k}
\bar{T}_\ell(x)\,(y_j-y_\ell).
\end{equation}
Consequently, letting $D_Y:=\max_{i,j\in [k]}\|y_i-y_j\|$, we have
$\big\|\nabla_x \bar{T}_j(x)\big\|\ \le\ \frac{2D_Y}{\lambda}$.
\end{proposition}

Please refer to Appendix~\ref{apx:proof-bounded-gradient}.
To conclude, the regularization parameter $\lambda$ controls the trade-off between the sensitivity and the optimality of the transport plan $\bar{T}$: larger $\lambda$ tightens the bound and enhances robustness to data perturbations, but introduces bias relative to the true OT plan, thereby the fidelity of data distribution, as discussed in Corollary \ref{cor:sinkhorn-lowbound}. By the same reasoning, uniform-weighted $k$-mean does not enjoy this property; for more stability results, please refer to Appendix \ref{AP:stable}.

\section{Numerical Experiments}\label{sec:experiments}

\paragraph*{Methods.}

1.\textbf{EMS}: Our proposed function space variational EM algorithm for weighted coreset construction with various $\lambda$;
2. \textbf{WCSL}: a gradient-based Sinkhorn coreset learner with \textit{uniform weights} following \cite{yin2025wasserstein}.
3. \textbf{K-mean centroids}: a mean squared error(MSE) coreset with \textit{uniform weights}\citep{lloyd1982least}.
Additional experimental details can be found in Appendix~\ref{AP:experiments}.

\begin{figure}[t]
  \centering
  \includegraphics[width=.82\linewidth]{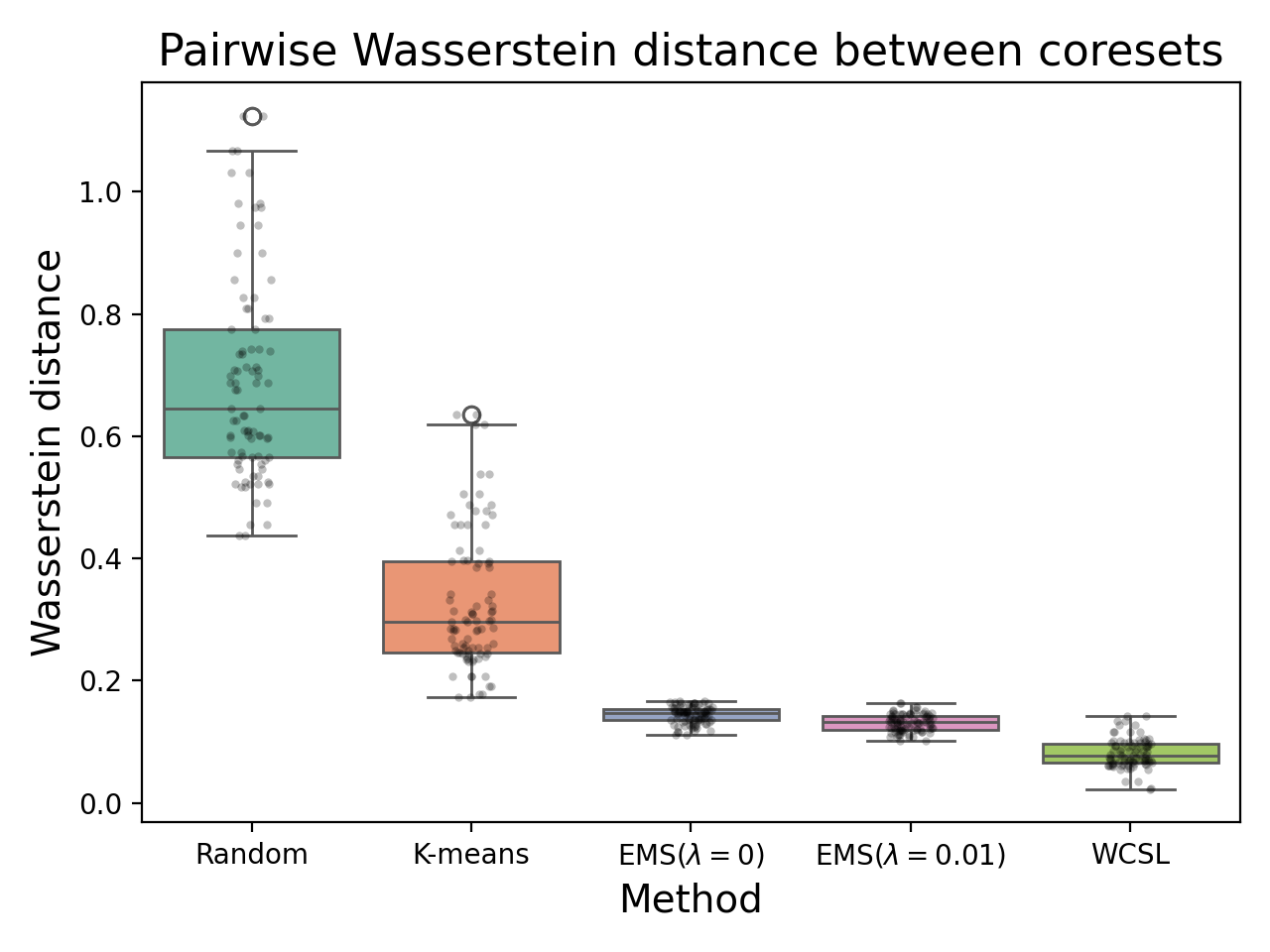}
  \includegraphics[width=.82\linewidth]{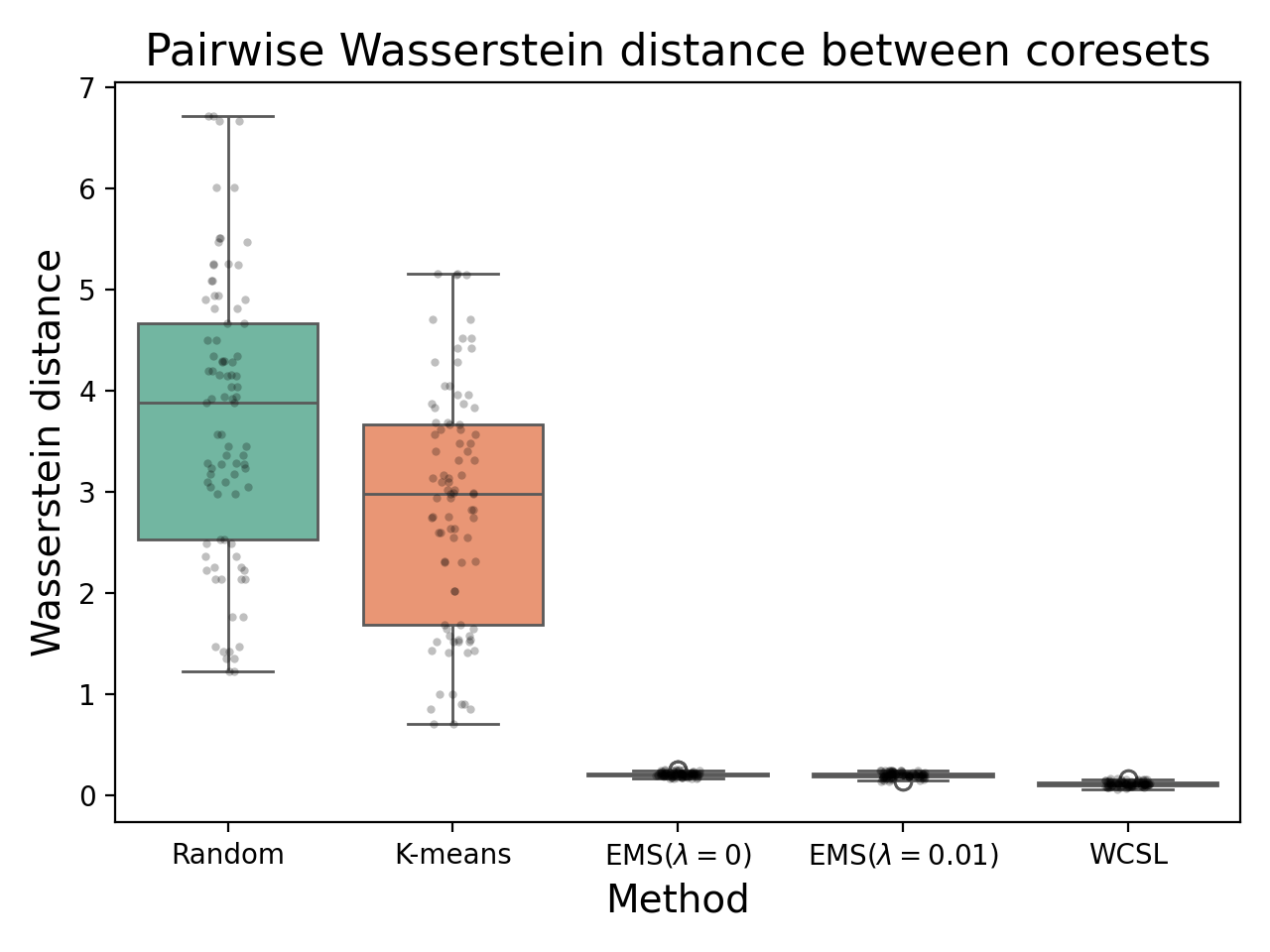}
  \caption{ Pairwise entropic Wasserstein distances across $10$ random initializations for Gaussian mixture dataset with $R=2$ (top) and $R=5$ (bottom).}
  \label{fig:stability}
  \vspace{-5pt}
\end{figure}
\subsection{Distributional Approximation Accuracy}\label{sec:distance}
This section we ask ``Does $\nu_k$ obtained by our algorithms approximate $\mu$ well in distributional distance, as suggested by Theorem~\ref{thm:finite-k-bound} and Corollary~\ref{thm:consistency}?"

\textbf{Setup.} We report entropic Wasserstein distances between the empirical dataset $\mu$ and $\nu_k$: (1) a $d{=}100$ Gaussian mixture benchmark with $n=10^4$ and $k\in\{10,20,50,100,200\}$, and (2) MNIST with the same range of $k$. We also record wall-clock time to show the quality-cost trade-off.

Table \ref{tab:distance-runtime} shows, EMS has similar performance in distributional accuracy as WCSL, while WCSL is much slower as it uses gradient descent to minimize the Sinkhorn loss directly. K-means lags on the image task, consistent with the uniform-weight bias in Lem.~\ref{cor:kmeans-inconsistent}. Additional experiments are conducted on real-world benchmark datasets as used in \citet{manousakas2020bayesian} can be found in table~\ref{tab:real_world}.

\subsection{Efficiency and Runtime Scaling}\label{sec:runtime}
Importantly, we ask: ``How fast is the EMS coreset construction compared to existing methods?"

\textbf{Setup.}
Gaussian mixture data is simulated as shown in Fig.~\ref{fig:gaussian}. Baseline is $n=10^{4},d=10, k=100$ and we vary one factor at a time:
\(n\in\{ 10^{4},10^{5},10^{6}\}\), \(d\in\{1,10,10^{2},10^{3}\}\), \(k\in\{10,10^{2},10^{3}\}\).
We report mean \(\pm\) std. wall-clock runtime over 20 repetitions.

Table~\ref{tab:runtime-summary} shows that the EMS family follows the near-linear cost via Alg.~\ref{alg:pop-streaming-em}: even at $n{=}10^6$ the unregularized EMS run finishes in less than ten seconds and the entropic variants remain in the same single-digit band, while WCSL stretches to thousands of seconds. The time complexity of our algorithm finding OT plan is $O(nkd)$, which is linear to the dataset size, coreset size and data dimension.

Increasing the dimension only causes a modest drift (roughly $20$ to $200$ ms), thanks to the decoupled softmax E-step in Eq.~\ref{eq:estep}. Furthermore, varying $k$ leaves the wall-clock nearly flat, consistent with the closed-form M-step updates of Prop.~\ref{prop:descent-T}. These results validate that EMS achieves fidelity on par with the EM routine, while attaining two to three orders of magnitude faster runtime compared to gradient-based baselines.

\begin{table}[t]
  \centering
  \caption{Wasserstein distance and runtime in seconds. Lower is better. Means over 10 runs. Default setup is $n=10^{4},d=10, k=100$.}
  \label{tab:runtime-summary}
  \small
  \setlength{\tabcolsep}{6pt}
  \resizebox{\textwidth}{!}{%
  \begin{tabular}{lcccccc}
    \toprule
    \multirow{3}{*}{Varying} & \multirow{3}{*}{Setup} & {WCSL} & \multicolumn{4}{c}{\textsc{EMS}} \\
    \cmidrule(lr){3-7}
    & & $\lambda=10^{-2}$& $\lambda=0$ & $\lambda=10^{-3}$ & $\lambda=10^{-2}$ & $\lambda=10^{-1}$ \\
    & & distance / time & distance / time & distance / time & distance / time & distance / time \\
    \midrule
    \multirow{3}{*}{\textbf{Data size $n$}}
    &$10^4$ & $5.203 / 9.749$ & $3.697 / \mathbf{0.028}$ & $3.709 / 0.031$ & $3.691 / 0.052$ & $\mathbf{3.670} / 0.129$ \\
    &$10^5$ & $\mathbf{3.634} / 161.777$ & $3.764 / \mathbf{0.452}$ & $3.765 / 0.524$ & $3.756 / 0.723$ & $3.755 / 0.971$ \\
    & $10^6$ & $0.154 / 2{,}006.306$ & $0.152 / \mathbf{7.368}$ & $0.152 / 9.051$ & $0.151 / 7.957$ & $\mathbf{0.151} / 9.653$ \\
    \midrule
     \multirow{4}{*}{\textbf{Dimension $d$}}
    &$1$ & $0.00791 / 2.450$ & $\mathbf{0.00171} / 0.020$ & $0.00199 / 0.024$ & $0.00244 / 0.024$ & $0.00543 / \mathbf{0.020}$ \\
    &$10$ & $4.891 / 8.909$ & $3.730 / \mathbf{0.026}$ & $3.730 / 0.026$ & $3.727 / 0.049$ & $\mathbf{3.688} / 0.130$ \\
    &$10^2$ & $78.400 / 23.061$ & $59.649 / \mathbf{0.046}$ & $59.638 / 0.055$ & $59.593 / 0.064$ & $\mathbf{59.466} / 0.186$ \\
    &$10^3$ & $830.629 / 61.522$ & $629.881 / 0.212$ & $629.901 / \mathbf{0.196}$ & $629.890 / 0.252$ & $\mathbf{629.760} / 0.516$ \\
    \midrule
    \multirow{3}{*}{\textbf{Coreset size $k$}}&
    $10$ & $14.639 / 2.509$ & $7.191 / \mathbf{0.026}$ & $6.520 / 0.030$ & $6.262 / 0.033$ & $\mathbf{5.622} / 0.036$ \\
    &$10^2$ & $4.994 / 8.866$ & $3.706 / \mathbf{0.026}$ & $3.709 / 0.029$ & $3.696 / 0.053$ & $\mathbf{3.674} / 0.121$ \\
    &$10^3$ & $2.148 / 210.962$ & $1.951 / \mathbf{0.029}$ & $1.952 / 0.031$ & $1.939 / 0.052$ & $\mathbf{1.871} / 0.232$ \\
    \bottomrule
  \end{tabular}%
  }
\end{table}




\subsection{Hyper-parameter Selection for $\lambda$}
We ask: ``How sensitive is EMS to the regularization parameter $\lambda$?"

Empirically, Tables~\ref{tab:distance-runtime} and \ref{tab:runtime-summary} show that EMS attains very similar distributional accuracy across a wide range of $\lambda$ (including $\lambda\in\{10^{-3},10^{-2},10^{-1}\}$ and the unregularized $\lambda{=}0$ variant), while runtime increases smoothly as $\lambda$ grows.

This behavior is consistent with the EM interpretation: $\lambda$ acts as a temperature in the softmax E-step (Eq.~\ref{eq:estep}), primarily controlling the \emph{sharpness} of responsibilities. The M-step (Def.~\ref{def:M-step}) re-normalizes these responsibilities and recomputes weighted centroids and mixture weights, so moderate changes in $\lambda$ tend to leave the converged coreset nearly invariant. In practice, we use $\lambda=10^{-2}$ as a default, set $\lambda=0$ for the fastest hard-assignment variant, and prefer larger $\lambda$ when additional robustness to perturbations is desired.

\subsection{Stability Across Initializations}\label{sec:stability-exp}
Next, we ask ``Do the Lipschitz properties of the transport plan and center updates translate into robust coreset constructions?"

\textbf{Setup.} For ten independent seeds, each method fits a coreset with $k{=}64$ on four-component Gaussian mixtures (with a small $n{=}1000$), as shown in Fig.\ref{fig:gaussian}. We compute pairwise entropic Wasserstein distances between the resulting $\nu_k$ from seeds and visualize the distances on bar plots.

Fig.~\ref{fig:stability} shows that EMS with $\lambda=0$ or $0.01$ produces concentrated points across seeds in both overlap ($R{=}2$) and separated ($R{=}5$) datasets. On the contrary, K-means significantly degrades on $R=5$ Gaussian mixture dataset, reflecting the inability of uniform weights to reassign mass across empty regions, as discussed Lem.~\ref{cor:kmeans-inconsistent}, where its hard assignments make it more susceptible to poor local minima. WCSL adapts locations but remains seed-sensitive. The low dispersion of EMS($\lambda=0.01$) aligns with our Lipschitz arguments and the monotone M-step discussed in Prop.~\ref{prop:descent-T}, indicating robust convergence under soft assignments.

\subsection{Downstream Interpretability via Shapley Values}\label{sec:shapley}
\begin{table}[t]
  \centering
  \caption{Shapley reconstruction error on two test datasets (lower is better; best in bold). Results average over five runs.}
  \label{tab:shapley}
  \footnotesize
  \setlength{\tabcolsep}{4pt}
  \begin{tabular}{lccc}
    \toprule
    \multicolumn{4}{l}{\textbf{Census} ($n=6{,}513$)} \\
        \midrule
    Method & MeanAE & MaxAE & BVE \\
    EMS($\lambda=0.01$) & $\mathbf{5.07\times 10^{-3}}$ & $5.40\times 10^{-2}$ & $3.31\times 10^{-2}$ \\
    EMS($\lambda=0$) & $8.10\times 10^{-3}$ & $8.87\times 10^{-2}$ & $3.73\times 10^{-2}$ \\
    K-means & $5.76\times 10^{-3}$ & $\mathbf{4.34\times 10^{-2}}$ & $4.33\times 10^{-2}$ \\
    WCSL & $5.52\times 10^{-3}$ & $4.92\times 10^{-2}$ & $\mathbf{1.01\times 10^{-2}}$ \\
    \midrule
    \multicolumn{4}{l}{\textbf{Diabetes} ($n=89$)} \\
        \midrule
    Method & MeanAE & MaxAE & BVE \\
    EMS($\lambda=0.01$) & $\mathbf{1.04}$ & $\mathbf{4.38}$ & $\mathbf{1.38}$ \\
    EMS($\lambda=0$) & $1.43$ & $5.10$ & $2.17$ \\
    K-means & $1.39$ & $7.66$ & $10.75$ \\
    WCSL & $1.44$ & $6.55$ & $5.24$ \\
    \bottomrule
  \end{tabular}
\end{table}

Shapley-value explainers (SHAP; \citealp{shapley1953value,lundberg2017unified}) require a \emph{background distribution} to marginalize missing features, and the quality of this background directly affects feature attributions.
To test whether a coreset preserves downstream interpretability, we replace the full training background with a weighted coreset $\nu_k$ (here $k{=}64$) and compute test-set SHAP values using \texttt{shap.TreeExplainer}.
We treat SHAP computed with the full training background as the reference, and report deviations of coreset-background SHAP from this reference (MeanAE/MaxAE over the SHAP matrix) along with base-value error (BVE); full setup and definitions are in Appendix~\ref{AP:experiments}. We evaluate our method against $k$-means as a dataset summarization baseline for downstream tasks.

Table~\ref{tab:shapley} shows that EMS with $\lambda{=}0.01$ achieves the lowest \emph{average} distortion (MeanAE) on both datasets and is best on all three metrics on Diabetes, indicating that weighted Sinkhorn coresets provide faithful SHAP backgrounds.
On Census, K-means attains a slightly smaller worst-case deviation (MaxAE) but has larger mean distortion and base-value drift, consistent with the uniform-weight density bias discussed in Lem.~\ref{cor:kmeans-inconsistent}.
WCSL yields a strong base-value match on Census, but (as shown in Table~\ref{tab:distance-runtime}) incurs substantially higher coreset construction cost.
Overall, the improvements from EMS align with our mechanism: \emph{adaptive weights} better match the data density, while \emph{soft assignments} reduce high-variance attribution artifacts near low-density regions.

\section{Conclusion}
This work presents a fast and robust approach for constructing weighted coresets via approximate OT plans under a comprehensive theoretical framework on OT-based coresets. Our approach, EMS, combines a closed-form softmax E-step with decoupled M-step updates, enabling fast and stable optimization of the Sinkhorn loss. By maintaining strong fidelity to the original data, the coresets support downstream interpretability, such as Shapley value computation, while running orders of magnitude faster than gradient-based methods. A discussion of limitations and future research directions is provided in Appendix~\ref{AP:limitations-images}.

\bibliographystyle{plainnat}
\bibliography{references}

\newpage
\appendix
\section{Proofs}

\subsection{Proof of Theorem~\ref{thm:finite-k-bound}}\label{apx:proof-finite-k-bound}
\begin{proof}
Let $\mu = \rho(x)\,dx$ be a probability measure on a compact convex domain $\Omega \subset \mathbb{R}^d$, with $\rho \in C^2(\Omega)$ and $\rho_{\min} > 0$, and let
$\nu_k^* = \arg\min_{\nu \in \mathbb{G}_k} W_{\lambda,p}(\mu,\nu)$
be the optimal $k$-point quantizer. Denote by $h_k \sim k^{-1/d}$ the typical Voronoi diameter.

By classical quantization results from \cite{graf2000foundations} and Zador's Theorem \citep{zador1963quantization}, the deterministic Voronoi quantizer $\nu_{k,vor}$ satisfies
\begin{equation}
W_p^p(\mu, \nu_{k,vor}) = O(k^{-1/d}),
\end{equation}
where the space is partitioned into Voronoi cells $V_j$ around points $y_{vor}$ such that each cell contains all points closer to its center than to any other. This provides an upper bound for $W_p(\mu, \nu_k^*)$.

To bound the KL divergence, we define a Gibbsized local coupling $\pi_{\tilde{T}}$ for each Voronoi cell $V_j$:
\begin{equation}
\pi_{\tilde{T}}(dx,dy) = \rho(x)\,q_\lambda(y|x)\,dx\ \text{with}\
q_\lambda(y|x) = \frac{\exp\!\big(-\|x-y\|^p/\lambda\big)}{\sum_{m=1}^k \exp\!\big(-\|x-y_m\|^p/\lambda\big)}.
\end{equation}
The logarithm of the partition function is:
\begin{equation}
\log Z_\lambda(x) = \log \sum_{m=1}^k \exp(-c_m/\lambda), \text{with} \ c_m := \|x-y_m\|^p,
\end{equation}
which acts as a soft-minimum of the local costs $\{c_m\}_{m=1}^k$. For small $\lambda$, expanding $\log Z_\lambda(x)$ gives
\begin{equation}
- \lambda \log \sum_m e^{-c_m/\lambda} = \min_m c_m + \frac{\lambda}{2} \sum_{m \neq j^*} e^{-(c_m - c_{j^*})/\lambda} + \cdots,
\end{equation}
where $j^* = \arg\min_m c_m$.
From this expansion, the Gibbs probabilities are dominated by the nearest center:
\begin{equation}\label{eq:qapprox}
q_\lambda(y_m|x) \approx
\begin{cases}
1 - \sum_{m \neq j^*} e^{-(c_m - c_{j^*})/\lambda}, & m = j^* \\
e^{-2(c_m - c_{j^*})/\lambda}, & m \neq j^*
\end{cases}.
\end{equation}
For a single Voronoi cell $V_j$, the KL divergence w.r.t.\ the independent measure $\mu \otimes \nu$ is
\begin{equation}
\mathrm{KL}(\pi_{\tilde{T}}\Vert \mu\otimes\nu)
= \sum_m \int_{V_j} \pi_{\tilde{T}}(x,y_m) \log \frac{\pi_{\tilde{T}}(x,y_m)}{\mu(V_j) \nu(y_m)} \, dx.
\end{equation}
Plugging in $\pi_{\tilde{T}}(x,y_m) = \rho(x) q_\lambda(y_m|x)$ and separating the entropy term:
\begin{equation}
\mathrm{KL}(\pi_{\tilde{T}}\Vert \mu\otimes\nu) = \sum_m \int_{V_j} \rho(x) q_\lambda(y_m|x) \log q_\lambda(y_m|x) \, dx + \sum_m \int_{V_j} \rho(x) q_\lambda(y_m|x) \log \frac{\rho(x)}{\mu(V_j) \nu(y_m)} \, dx.
\end{equation}
The first term dominates for small $\lambda$. Using the approximation in \ref{eq:qapprox}, we obtain
\begin{equation}
\lambda \, \mathrm{KL}(\pi_{\tilde{T}}\Vert \mu\otimes\nu)  = \frac{\lambda}{2} \int_{V_j} \sum_{m \neq j^*} \frac{\rho(x)}{\mu(V_j)} e^{-(c_m - c_{j^*})/\lambda} \, dx + O(\lambda e^{-3 c h_k^p/\lambda}).
\end{equation}
Since $c_m(x) - c_{j^*}(x) = \Theta(h_k^p)$ within each cell, Laplace's method gives
\begin{equation}
\lambda \, \mathrm{KL}(\pi_{\tilde{T}}\Vert\mu\otimes\nu) = O\!\big(\lambda e^{-c h_k^p/\lambda}\big) + O(\lambda h_k^{2p}).
\end{equation}
Choosing $\lambda_k = c\, h_k^p = c\, k^{-p/d}$, the Taylor remainder is $O(k^{-2p/d})$, so
\begin{equation}
\mathcal{F}_{\lambda_k}(Y_{vor}, W_{vor}, \tilde{T}) = O(k^{-p/d}) + O(k^{-2p/d}).
\end{equation}
By optimality of $\nu_k^*$ and nonnegativity of KL,
\begin{equation}
W_p^p(\mu, \nu_k^*) \le W_{\lambda,p}^p(\mu, \nu_k^*) \le C k^{-p/d} + C' k^{-2p/d}.
\end{equation}
Taking the $p^\text{th}$ root, we conclude
\begin{equation}
W_p(\mu, \nu_k^*) \le C k^{-1/d} + C' k^{-2/d},
\end{equation}
where the first term is the classical quantization rate, and the second arises from the second-order Taylor expansion of the entropic bias. This corresponds to Eq.\ref{eq:finite-m-bound} and thus completes the proof.

By the way, if we further assume that the support of $\mu$ lies on a $d$-dimensional smooth manifold $\mathcal{M} \subset \Omega$ with $d_m=d/2$ intrinsic manifold dimension, the quantization error $C k^{-1/d_m}$ in the first term can be absorbed into the second-order term by suitably adjusting constants, under standard manifold quantization argument \citep{graf2000foundations}. The Wasserstein distance admits the simplified bound
\begin{equation}
W_p(\mu, \nu_k^*) \le 2 C k^{-2/d},
\end{equation}
where $C$ depends on the manifold geometry and smoothness of $\rho$, $p$, and data dimension $d$.
\end{proof}

\subsection{Proof of Corollary~\ref{thm:consistency}}\label{apx:proof-consistency}

\begin{proof}
Taking the limit as $k \to \infty$, since $1/d > 0$ and $2/d > 0$, we have
\begin{equation}
\lim_{k \to \infty} C\, k^{-1/d} = 0, \quad \lim_{k \to \infty} C'\, k^{-2/d} = 0.
\end{equation}
Thus, combining the vanishing of the classical quantization error and the entropic KL correction, we conclude that, asymptotically,
\begin{equation}
\lim_{k \to \infty} W_p(\mu, \nu_k^*) = 0.
\end{equation}
In other words, the Sinkhorn coresets converge to the underlying $\mu$ in terms of the Wasserstein distance.
\end{proof}
%
%

\subsection{High-resolution asymptotics for uniform $k$-means}\label{AP:highres}

We first describe the regularity assumption and classical density limit that is mentioned Sec.~\ref{subsec:kmeans-inconsistent}. Lemma~\ref{cor:kmeans-inconsistent} is an informal version of Theorem~\ref{thm:kmeans-density}.

\begin{assumption}[High-resolution regularity\citep{gersho1979asymptotically}]
\label{asmp:highres}
Let $\mu$ be absolutely continuous on $\R^d$ with density $p$ that is bounded and continuous on a compact set. For each $k$, let $Y_k=\{y_1,\ldots,y_k\}\subset\R^d$ be a (near-)minimizer of the distortion
\begin{equation}
\begin{aligned}
L_k(Y)
&=\int_{\R^d}\min_{1\le j\le k}\|x-y_j\|^2\,p(x)\,dx.
\end{aligned}
\end{equation}
Write $\{V_j\}_{j=1}^k$ for the Voronoi partition induced by $Y_k$. We assume the usual high-resolution regularity: cell diameters $\max_j \operatorname{diam}(V_j)\to 0$ as $k\to\infty$, boundary effects are negligible, and local cell shapes are asymptotically regular.
\end{assumption}

\begin{theorem}[Asymptotic density for $k$-means\citep{gersho1979asymptotically}]\label{thm:kmeans-density}
Under Assumption~\ref{asmp:highres}, any sequence of asymptotically optimal centroids $\{Y_k\}$ admits a limiting \emph{codepoint density} $g^\star$ on $\R^d$ such that, for every bounded continuous test function $\varphi$ as $k\to\infty$, we have
\begin{equation}
\label{eq:kmeans-dd2}
\frac{1}{k}\sum_{j=1}^k \varphi(y_j)
\longrightarrow \int_{\R^d}\varphi(y)\,g^\star(y)\,dy,
\text{with}\
g^\star(y)
\propto p(y)^{\frac{d}{d+2}}.
\end{equation}
Equivalently, $g^*(x)$ can't recover underlying true $p(x)$ and
\begin{equation}
g^\star \;=\; \frac{p^{\,\frac{d}{d+2}}}{\int p^{\,\frac{d}{d+2}}}.
\end{equation}
\end{theorem}

\begin{proof}
Partition $\R^d$ into Voronoi cells $\{V_j\}_{j=1}^k$ induced by the generator set $Y=\{y_j\}_{j=1}^k$.
The expected quantization error can be expressed as
\begin{equation}
L_k(Y) = \sum_{j=1}^k \int_{V_j} \|x - y_j\|^2\, p(x)\, dx.
\end{equation}
Under the local uniformity assumption on $p$ within each $V_j$, we have the asymptotic relation
\begin{equation}
L_k(Y) = \sum_{j=1}^k p(y_j) \int_{V_j} \|x - y_j\|^2\, dx + o(k^{-2/d}),
\quad \text{as } k \to \infty.
\end{equation}
From high-resolution quantization theory, each cell integral satisfies
\begin{equation}
\int_{V_j} \|x - y_j\|^2\, dx
= \alpha_d\, |V_j|^{1+\frac{2}{d}} + o(|V_j|^{1+\frac{2}{d}}),
\end{equation}
where $\alpha_d > 0$ is a shape-dependent constant.
If the empirical locations $\{y_j\}$ are asymptotically distributed according to a continuous density $g$, then the cell volumes satisfy
\begin{equation}
|V_j|\, g(y_j) = k^{-1} + o(k^{-1}).
\end{equation}
Substituting these expressions yields
\begin{equation}
L_k(Y) = \alpha_d\, k^{-\frac{2}{d}}
\sum_{j=1}^k p(y_j)\, g(y_j)^{-\frac{2}{d}}\, |V_j|
+ o(k^{-\frac{2}{d}}),
\end{equation}
which in the limit becomes
\begin{equation}
L_k(Y) = \alpha_d\, k^{-\frac{2}{d}}
\int_{\R^d} p(y)\, g(y)^{-\frac{2}{d}}\, dy
+ o(k^{-\frac{2}{d}}).
\end{equation}
Minimizing the leading-order term over densities $g$ under the constraint $\int g = 1$
and applying H\"older's inequality yields the unique optimizer
\begin{equation}
g^\star(y) \propto p(y)^{\frac{d}{d+2}} \neq p(y).
\end{equation}
Finally, by a Riemann-sum argument under Assumption~\ref{asmp:highres},
the discrete sum converges weakly to the integral in equation~\ref{eq:kmeans-dd2}.
\end{proof}

\subsection{Proof of Proposition~\ref{prop:weights-fix}}\label{apx:proof-weights-fix}
\begin{proof}
For any bounded continuous $\varphi$,
\begin{equation}
\sum_{j=1}^k \omega_j \varphi(y_j)\approx
\frac{\sum_{j=1}^k \frac{p(y_j)}{g^\star(y_j)}\,\varphi(y_j)}
{\sum_{j=1}^k \frac{p(y_j)}{g^\star(y_j)}}\longrightarrow
\frac{\int \frac{p(y)}{g^\star(y)}\varphi(y)\,g^\star(y)\,dy}
{\int \frac{p(y)}{g^\star(y)}\,g^\star(y)\,dy}=\int \varphi(y)\,p(y)\,dy,
\end{equation}
using weak convergence of $\{y_j\}$ to $g^\star$ and continuity of $y\mapsto p(y)/g^\star(y)$. Hence $\nu_k^{\,w}$ weakly converge to $\mu$ as $k\to\infty$.
\end{proof}

\subsection{Proof of Lemma~\ref{lemma:Fenchel}}\label{apx:proof-fenchel}
\begin{proof}
For any $f_1, f_2$ and $\lambda \in [0,1]$, by Hölder's inequality
\begin{equation}
F_{\lambda f_1 + (1-\lambda)f_2,P_0}
= \log \mathbb{E}_{P_0}\!\left[e^{\lambda f_1(T)+(1-\lambda)f_2(T)}\right] \le \log \Big(\mathbb{E}_{P_0}[e^{f_1(T)}]\Big)^{\lambda}
           \Big(\mathbb{E}_{P_0}[e^{f_2(T)}]\Big)^{1-\lambda}= \lambda F_{f_1,P_0} + (1-\lambda)F_{f_2,P_0}.
\end{equation}

Hence, $F_{f,P_0}$ is convex in $f$. For any $q \ll P_0$, let $r(T) = \frac{dq}{dP_0}(T)$, then
\begin{equation}
\mathbb{E}_{q}[f(T)] = \mathbb{E}_{P_0}[r(T)f(T)], \quad
\mathrm{KL}(q\|P_0) = \mathbb{E}_{P_0}[r(T)\log r(T)].
\end{equation}
By definition of $F_{f,P_0}$,
\begin{equation}
F_{f,P_0} = \log \mathbb{E}_{P_0}[e^{f(T)}]
          = \log \mathbb{E}_{q}\!\left[\frac{e^{f(T)}}{r(T)}\right].
\end{equation}
Applying Jensen’s inequality to $\log$ yields
\begin{equation}
F_{f,P_0}
\ge \mathbb{E}_{q}\!\left[\log \frac{e^{f(T)}}{r(T)}\right]
 = \mathbb{E}_{q}[f(T)] - \mathbb{E}_{q}[\log r(T)] = \mathbb{E}_{q}[f(T)] - \mathrm{KL}(q\|P_0).
\end{equation}
Taking the supremum over all $q \ll P_0$ gives
\begin{equation}
F_{f,P_0} \ge \sup_{q \ll P_0} \Big\{\mathbb{E}_{q}[f(T)] - \mathrm{KL}(q\|P_0)\Big\}.
\end{equation}

To understand the equality condition, define the Gibbs measure
\begin{equation}
q^*(T) = \frac{e^{f(T)} P_0(T)}{\mathbb{E}_{P_0}[e^{f(T)}]}.
\end{equation}
Then  and
\begin{align}
\frac{dq^*}{dP_0}(T) = \frac{e^{f(T)}}{\mathbb{E}_{P_0}[e^{f(T)}]}; \quad \mathrm{KL}(q^*\|P_0)
&= \mathbb{E}_{q^*}\!\left[\log \frac{dq^*}{dP_0}\right]
 = \mathbb{E}_{q^*}[f(T)] - F_{f,P_0}.
\end{align}
Rearranging gives
\begin{equation}
F_{f,P_0} = \mathbb{E}_{q^*}[f(T)] - \mathrm{KL}(q^*\|P_0),
\end{equation}
where equality holds at $q = q^*$. Therefore,
\begin{equation}\label{eq:fenchel}
F_{f,P_0}
= \sup_{q \ll P_0} \Big\{\mathbb{E}_{q}[f(T)] - \mathrm{KL}(q\|P_0)\Big\},
\end{equation}
with unique maximizer $q^*(T) \propto e^{f(T)} P_0(T)$.
\end{proof}

\subsection{Proof of Corollary~\ref{cor:sinkhorn-lowbound}}\label{apx:proof-sinkhorn-lowbound}
\begin{proof}
Starting from \ref{eq:fenchel} at the optimum $q^*$,
\begin{equation}
F_{f,P_0}
= \mathbb{E}_{T\sim q^*}[f(T)] - \mathrm{KL}(q^*\|P_0).
\end{equation}

Let $f(T) = -\tfrac{1}{\lambda}\langle C, T\rangle_\mu$ and multiply both sides by $-\lambda$,
\begin{align}\label{eq:freeenergy-expand}
-\lambda F_{f,P_0}
= \mathbb{E}_{T\sim q^*}[\langle C,T\rangle_\mu]
  + \lambda\,\mathrm{KL}(q^*\|P_0).
\end{align}

By assumption on $P_0$, $
\mathrm{KL}(q^*\|P_0)
= \mathbb{E}_{T\sim q^*}\!\Big[\KL(\pi_T\|\mu\otimes\nu_k)\Big]$, then substituting into~\ref{eq:freeenergy-expand} gives
\begin{equation}
-\lambda F_{f,P_0}
= \mathbb{E}_{T\sim q^*}[\langle C,T\rangle_\mu]
  + \lambda\,\mathbb{E}_{T\sim q^*}\!\Big[\KL(\pi_T\|\mu\otimes\nu_k)\Big].
\label{eq:freeenergy-sinkhorn}
\end{equation}

Since both $\langle C, T\rangle_\mu$ and $\KL(\pi_T\|\mu\otimes\nu_k)$ are convex in $T$,
by Jensen’s inequality,
\begin{align}
\mathbb{E}_{T\sim q^*}[\langle C,T\rangle_\mu]
&\ge \langle C,\bar{T}\rangle_\mu,\\
\mathbb{E}_{T\sim q^*}\!\Big[\KL(\pi_T\|\mu\otimes\nu_k)\Big]
&\ge \KL(\pi_{\bar{T}}\|\mu\otimes\nu_k),
\end{align}
where $\bar{T} := \mathbb{E}_{T\sim q^*}[T]$. Applying these to~\ref{eq:freeenergy-sinkhorn} yields the inequality
\begin{equation}
-\lambda F_{f,P_0}
\ge \langle C,\bar{T}\rangle_\mu
  + \lambda\,\KL(\pi_{\bar{T}}\|\mu\otimes\nu_k).
\end{equation}
The right-hand side corresponds to the regularized optimal transport objective,
\begin{equation}
\mathcal{F}_\lambda(Y,W,\bar{T})
= \langle C,\bar{T}\rangle_\mu
  + \lambda\,\KL(\pi_{\bar{T}}\|\mu\otimes\nu_k),
\end{equation}
whose minimum over $\mathbb{T}_k$ defines the Sinkhorn loss.
Hence,
\begin{equation}
-\lambda F_{f,P_0}
\ge \mathcal{F}_\lambda(Y,W,\bar{T})
\ge \inf_{T\in\mathbb{T}_k} \mathcal{F}_\lambda(Y,W,T).
\end{equation}

\paragraph*{(A) Equality in the Jensen step.}
The first inequality in Corollary~\ref{cor:sinkhorn-lowbound}
originates from Jensen’s inequality, which is tight if and only if $T$ is almost surely constant under $q^*$, i.e.\ $q^*=\delta_T$.
This occurs either when $P_0=\delta(T-T_0)$ or in the zero-temperature limit $\lambda\to0$,
where the Gibbs measure collapses to a single deterministic transport plan:
\begin{equation}
\lim_{\lambda\to0} q^*(T) = \delta_{T^*}(T).
\end{equation}
In this limit the entropy term $\KL(q^*\|P_0)$ vanishes and the free energy reduces
to a pure energy functional:
\begin{equation}
-\lambda F_{f,P_0} = \langle C,T^*\rangle_\mu.
\end{equation}
Geometrically, the collapse of $q^*$ corresponds to the elimination of stochasticity
in the transport plan, while probabilistically it represents the transition from a
Boltzmann distribution to a deterministic mapping.

\paragraph*{(B) Equality at the global minimum.}
The second inequality,
\begin{equation}
\mathcal{F}_\lambda(Y,W,\bar{T}) \ge \inf_T \mathcal{F}_\lambda(Y,W,T),
\end{equation}
is tight when $\bar{T}=\mathbb{E}_{T\sim q^*}[T]$ coincides with the global minimizer
$T^*$ of the Sinkhorn functional. Even if $q^*$ is not a Dirac measure,
equality is approximately achieved when $q^*$ concentrates its mass near $T^*$,
so that $\mathbb{E}_{q^*}[T]\approx T^*$.
In this case, the Gibbs measure $q^*$ behaves as a Boltzmann distribution over
transport plans, assigning higher probability to low-cost couplings.
As the temperature $\lambda$ decreases, $q^*$ becomes increasingly peaked around $T^*$,
and the free-energy bound becomes tight:
\begin{equation}
-\lambda F_{f,P_0} \;\to\; \mathcal{F}_\lambda(Y,W,T^*).
\end{equation}
\end{proof}

\subsection{Proof of Theorem~\ref{thm:estep-factorizes}}\label{apx:proof-estep}
\begin{proof}
From Lemma~\ref{lemma:Fenchel}, the optimal variational distribution, when $
f(T) = -\tfrac{1}{\lambda}\langle C,T\rangle_\mu$, is
\begin{equation}
q^*(T)
= \frac{e^{-\frac{1}{\lambda}\langle C,T\rangle_\mu} P_0(T)}%
{\E_{T\sim P_0}[e^{-\frac{1}{\lambda}\langle C,T\rangle_\mu}]}.
\end{equation}
By assumption, $P_0$ factorizes over $x\in\R^d$ with local components $P_{0,x}$,
that is,
\begin{equation}
P_0(T) = \prod_{x\in\R^d} P_{0,x}(T(x)),
\end{equation}
where each $T(x)\in\Delta_k$.
Since the exponential term also factorizes across $x$ as
\begin{equation}
e^{-\frac{1}{\lambda}\langle C,T\rangle_\mu}
= \prod_{x\in\R^d} e^{-\frac{1}{\lambda}\sum_{j=1}^k C_j(x)T_j(x)},
\end{equation}
this together implies that $q^*$ inherits
the same factorization:
\begin{equation}
    \begin{aligned}
        q^*(T)
&= \prod_{x\in\R^d} q^*_x(T(x));\\
q^*_x(T(x))& \propto P_{0,x}(T(x))\, e^{-\frac{1}{\lambda}\sum_{j=1}^k C_j(x)T_j(x)}.
    \end{aligned}
\end{equation}
To verify that it is indeed a sufficient condition, the Kullback–Leibler divergence decomposes additively:
\begin{equation}
\KL(q^*\|P_0)
=\sum_{x\in\R^d}\KL(q^*_x\|P_{0,x})
=\E_{q^*}\!\Big[\KL(\pi_T\|\mu\otimes\nu_k)\Big].
\end{equation}
Finally, if $P_0$ factorizes with weights $\omega=(\omega_1,\dots,\omega_k)$ such that
$P_{0,x}(T(x)) = \sum_j \omega_j \,\delta_{e_j}(T(x))$, then each local factor $q^*_x$ is categorical with unnormalized weights
$\omega_j\, e^{-C_j(x)/\lambda}$.
Normalizing yields the softmax form:
\begin{equation}
\bar{T}_j(x)
= \frac{\omega_j \, e^{-C_j(x)/\lambda}}
       {\sum_{l=1}^k \omega_l \, e^{-C_l(x)/\lambda}},
\qquad j=1,\dots,k,
\end{equation}
which matches Equation \ref{eq:estep}.
\end{proof}
The Mean-Field assumption has a clear probabilistic and geometric interpretation:

\paragraph*{(A) Mean-field factorization and conditional independence.}
Assuming $P_0$ factorizes over $x$ imposes a \emph{mean-field approximation}:
each spatial location $x$ contributes independently to the transport plan.
Consequently, the optimal Gibbs posterior $q^*$ inherits the same factorized form,
decoupling the optimization problem across all $x$.
This drastically reduces complexity and aligns the free-energy functional
with a sum of independent local free energies.

\paragraph*{(B) Softmax as the variational expectation.}
The expected plan $\bar{T}(x)$ emerges as a \emph{vector-valued softmax},
analogous to the $E$-step in the EM algorithm or the soft assignment in
Sinkhorn-based clustering.
Each $\bar{T}_j(x)$ represents the posterior probability that mass at $x$
is transported to cluster $j$, weighted by the prior $\omega_j$
and penalized by the transport cost $C_j(x)$.
The temperature $\lambda$ controls the sharpness of this assignment:
as $\lambda\to0$, $\bar{T}(x)$ approaches a hard assignment,
recovering the deterministic optimal transport map.

\paragraph*{(C) Relation to variational inference.}
Under this factorization, minimizing the free-energy functional
\begin{equation}
-\lambda F_{f,P_0}
= \E_{T\sim q^*}[\langle C,T\rangle_\mu]
+ \lambda \E_{T\sim q^*}[\KL(\pi_T\|\mu\otimes\nu_k)]
\end{equation}
corresponds to performing an $E$-step in a mean-field variational EM scheme:
$q^*$ plays the role of the posterior over transport assignments,
and $\bar{T}$ acts as its expected sufficient statistic.

Theorem~\ref{thm:estep-factorizes} provides a rigorous link between
the free-energy formulation of regularized optimal transport and the classical
mean-field variational inference paradigm.
The softmax update in~\eqref{eq:estep} serves as the closed-form
$E$-step, bridging Gibbs variational principles, entropy-regularized transport,
and probabilistic clustering.

\subsection{Proof for Proposition~\ref{prop:descent-T}}\label{AP:Proof_Estep}
\begin{proof}This proof elaborates Proposition~\ref{prop:descent-T} from Section~\ref{gen_inst}. Let $\bar{T}^{(t+1)}$ be fixed from the E-step. Consider Sinkhorn Loss $\mathcal{F}_\lambda(Y, W, \bar{T}^{(t+1)}).$

For each cluster $j$, the M-step optimizes
\begin{equation}
\begin{aligned}
\langle c(x,y_j),\bar{T}^{(t+1)}_j \rangle_\mu
&= \int_{\mathbb{R}^d} \bar{T}^{(t+1)}_j(x) \, \|x - y_j\|^2 \, d\mu(x)
\end{aligned}
\end{equation}
with respect to $y_j$. Differentiating and setting the gradient to zero gives
\begin{equation}
\begin{aligned}
2 \Big( y_j \int \bar{T}^{(t+1)}_j(x) d\mu(x)
- \int x \, \bar{T}^{(t+1)}_j(x) \, d\mu(x) \Big)
&= 0,
\end{aligned}
\end{equation}
which yields the update of $Y^{(t+1)}$
\begin{equation}
y_j^{(t+1)} = \frac{\int x \, \bar{T}^{(t+1)}_j(x) \, d\mu(x)}{\int \bar{T}^{(t+1)}_j(x) \, d\mu(x)}.
\end{equation}

Similarly, the weight $\omega_j$ appears in the linear term of the KL divergence between the variational transport plan $\bar{T}^{(t+1)}$ and the discrete measure $\mu \otimes \nu_k$:
\begin{equation}
\mathrm{KL}(\pi_{\bar{T}^{(t+1)} }\,\|\, \mu \otimes \nu_k)
= \sum_{j=1}^k \int_{\mathbb{R}^d} \bar{T}^{(t+1)}_j(x)
\log \frac{\bar{T}^{(t+1)}_j(x)}{\omega_j} \, d\mu(x).
\end{equation}
Differentiating with respect to $\omega_j$ and setting the derivative to zero gives
\begin{equation}
\frac{\partial}{\partial \omega_j} \mathrm{KL}(\bar{T}^{(t+1)} \,\|\, \mu \otimes \nu_k)
= - \frac{1}{\omega_j} \int_{\mathbb{R}^d} \bar{T}^{(t+1)}_j(x) \, d\mu(x) = 0,
\end{equation}
which yields the closed-form update of $W:=\{w_j\}_{j=1}^k$. The total mass assigned to cluster $j$ for the next step is
\begin{equation}
\omega_j^{(t+1)} = \int_{\mathbb{R}^d} \bar{T}^{(t+1)}_j(x) \, d\mu(x).
\end{equation}

By definition, the M-step performs a block coordinate optimization over $Y$ and $W$:
\begin{align}
\mathcal{F}_\lambda(Y^{(t+1)}, W^{(t)}, \bar{T}^{(t+1)}) &\le \mathcal{F}_\lambda(Y^{(t)}, W^{(t)}, \bar{T}^{(t+1)}), \\
\mathcal{F}_\lambda(Y^{(t+1)}, W^{(t+1)}, \bar{T}^{(t+1)}) &\le \mathcal{F}_\lambda(Y^{(t+1)}, W^{(t)}, \bar{T}^{(t+1)}),
\end{align}
which immediately implies
\begin{equation}
\mathcal{F}_\lambda(Y^{(t+1)}, W^{(t+1)}, \bar{T}^{(t+1)}) \le \mathcal{F}_\lambda(Y^{(t)}, W^{(t)}, \bar{T}^{(t+1)}).
\end{equation}

Equality holds if and only if $(Y^{(t)}, W^{(t)})$ are already minimizers along both blocks, i.e., they satisfy the first-order optimality conditions:
\begin{equation}
\begin{aligned}
\nabla_Y \mathcal{F}_\lambda(Y^{(t)}, W^{(t)}, \bar{T}^{(t+1)}) &= 0, \\
\nabla_W \mathcal{F}_\lambda(Y^{(t)}, W^{(t)}, \bar{T}^{(t+1)}) &= 0.
\end{aligned}
\end{equation}

This completes the proof.
\end{proof}

\subsection{Proof of Proposition~\ref{prop:bounded-gradient}}\label{apx:proof-bounded-gradient}
\begin{proof}
Throughout this proof we use the squared--Euclidean cost
\(C_j(x)=\|x-y_j\|^2\) (i.e., \(p=2\)), consistent with the E--step and Alg.~1.
Recall the soft assignment from Eq.~\eqref{eq:estep}:
\begin{equation}
    \begin{aligned}
\bar{T}_j(x)
&=\frac{\omega_j\,e^{-C_j(x)/\lambda}}{\sum_{\ell=1}^k \omega_\ell\,e^{-C_\ell(x)/\lambda}}
=\frac{a_j(x)}{Z(x)},\\
a_j(x)&:=\omega_j e^{-\|x-y_j\|^2/\lambda},\quad
Z(x):=\sum_{\ell=1}^k a_\ell(x).
\end{aligned}
    \end{equation}
Since
\begin{equation}
\begin{aligned}
    \nabla_x a_j(x)&=a_j(x)\,\nabla_x\!\big(-\|x-y_j\|^2/\lambda\big)
=-(2/\lambda)\,a_j(x)\,(x-y_j);\\
\nabla_x Z(x)&=\sum_{\ell}\nabla_x a_\ell(x)=-(2/\lambda)\sum_{\ell}a_\ell(x)\,(x-y_\ell)   ,
\end{aligned}
\end{equation}
Applying the product and chain rules gives
\begin{equation}
\begin{aligned}
\nabla_x \bar{T}_j(x)
&=\frac{\nabla_x a_j}{Z}-\frac{a_j}{Z^2}\,\nabla_x Z=-\frac{2}{\lambda}\frac{a_j}{Z}(x-y_j)+\frac{2}{\lambda}\frac{a_j}{Z}\sum_{\ell=1}^k\frac{a_\ell}{Z}(x-y_\ell)
\\
&=\frac{2}{\lambda}\,\bar{T}_j(x)\sum_{\ell=1}^k \bar{T}_\ell(x)\,(y_j-y_\ell).
\end{aligned}
\end{equation}
This proves the stated gradient identity, i.e. equation~\ref{eq:bounded-gradient} in main manuscript. For the norm bound, let
\(D_Y:=\max_{r,s}\|y_r-y_s\|\). Then
\begin{equation}
\Big\|\sum_{\ell=1}^k \bar{T}_\ell(x)\,(y_j-y_\ell)\Big\|\le \sum_{\ell=1}^k \bar{T}_\ell(x)\,\|y_j-y_\ell\|\le D_Y\sum_{\ell=1}^k \bar{T}_\ell(x)=D_Y,
    \end{equation}
and hence \(\|\nabla_x \bar{T}_j(x)\|\le (2/\lambda)\,\bar{T}_j(x)\,D_Y\le 2D_Y/\lambda\).
Note the weights \(\omega\) enter only through \(\bar{T}\) and cancel in the
derivative formula; the bound therefore holds for non‑uniform \(W\).
\end{proof}

\subsection{Discussion for Stability}\label{AP:stable}
In this subsection we collect Lipschitz‑type stability properties of the soft assignment map
\begin{equation}
\bar{T}^{(\lambda)}_j(x;Y,W)=\frac{\omega_j\,\exp\!\big(-\|x-y_j\|^2/\lambda\big)}{\sum_{\ell=1}^k \omega_\ell\,\exp\!\big(-\|x-y_\ell\|^2/\lambda\big)}.
\end{equation}
For brevity we write \(\bar{T}_j(x)\) when the dependence on \((Y,W,\lambda)\) is clear.

\begin{theorem}[Lipschitz continuity in the data]\label{thm:lip-x-T}
For any \(x,x'\in\R^d\) and every \(j\in[k]\),
\begin{equation}
\big|\bar{T}_j(x)-\bar{T}_j(x')\big|
\ \le\ \frac{2D_Y}{\lambda}\,\|x-x'\|.
\end{equation}
\end{theorem}
\begin{proof}
By Proposition~\ref{prop:bounded-gradient}, \(\|\nabla_x \bar{T}_j(x)\|\le 2D_Y/\lambda\)
for all \(x\). Apply the mean‑value theorem along the segment
\(x(t)=x+t(x'-x)\), \(t\in[0,1]\), and integrate over $x$, which completes the proof.
\end{proof}

\begin{proposition}[Partial derivatives with respect to centers]\label{prop:partial-y}
Fix \(x\in\R^d\). For any \(j,k\in[k]\),
\begin{equation}
\partial_{y_k}\bar{T}_j(x)
=\frac{2}{\lambda}\,\bar{T}_j(x)\Big[(x-y_j)\,\mathbf{1}\{j=k\}
-\bar{T}_k(x)\,(x-y_k)\Big].
\end{equation}
Consequently, if \(x\) ranges in a compact set \(K\subset\R^d\), then for all \(x\in K\),
\begin{equation}
\big\|\partial_{y_k}\bar{T}_j(x)\big\|\ \le\ \frac{2}{\lambda}\,(R_X+R_Y),
\end{equation}
where \(R_X:=\sup_{x\in K}\|x\|\), \(R_Y:=\max_{1\le \ell\le k}\|y_\ell\|\).
\end{proposition}
\begin{proof}
Let $\bar{T}_j(x) = a_j(x)/Z(x)$, where
\[
a_j(x) = \omega_j \exp\!\big(-\|x - y_j\|^2 / \lambda\big),
\quad\text{and}\quad
Z(x) = \sum_{\ell} a_\ell(x).
\]
Observe that only the term $a_k$ depends on $y_k$. Differentiating $a_k$ with respect to $y_k$ gives
\[
\partial_{y_k} a_k(x)
= a_k(x)\, \partial_{y_k}\!\big(-\|x - y_k\|^2 / \lambda\big)
= \frac{2}{\lambda}\, a_k(x)\, (x - y_k).
\]
Consequently,
\[
\partial_{y_k} Z(x) = \frac{2}{\lambda}\, a_k(x)\, (x - y_k).
\]
Using the quotient rule for differentiation of $\bar{T}_j = a_j / Z$, we obtain
\[
\partial_{y_k} \bar{T}_j(x)
= \frac{Z(x)\,\partial_{y_k} a_j(x) - a_j(x)\,\partial_{y_k} Z(x)}{Z(x)^2}.
\]
Since $a_j$ depends on $y_k$ only when $j = k$, we have
\[
\partial_{y_k} a_j(x)
= \begin{cases}
\dfrac{2}{\lambda}\, a_k(x)\, (x - y_k), & \text{if } j = k, \\[6pt]
0, & \text{if } j \ne k.
\end{cases}
\]
Substituting these expressions yields the stated identity for $\partial_{y_k} \bar{T}_j(x)$.

Moreover, if $x \in K$, then $\|x - y_\ell\| \le R_X + R_Y$ for all $\ell$, and by construction each $\bar{T}_j(x)$ satisfies $0 \le \bar{T}_j(x) \le 1$.
These bounds, together with the expression for $\partial_{y_k} \bar{T}_j(x)$, yield the claimed uniform estimate.
\end{proof}

\begin{theorem}[Lipschitz continuity in the centers on compact domains]\label{thm:lipschitz-centres}
Let \(K\subset\R^d\) be compact, \(x\in K\), and \(Y'=\{y'_1,\ldots,y'_k\}\).
Then for every \(j\in[k]\),
\begin{equation}
\big|\bar{T}_j(x;Y,W)-\bar{T}_j(x;Y',W)\big|
\ \le\ \frac{2k}{\lambda}\,\big(R_X+R_Y\big)\,\Delta_Y,
\end{equation}
where \(\Delta_Y:=\max_{1\le \ell\le k}\|y_\ell-y'_\ell\|\).
\end{theorem}
\begin{proof}
Consider the path \(Y(t)=\{y_\ell(t)\}\) with \(y_\ell(t)=y_\ell+t(y'_\ell-y_\ell)\), \(t\in[0,1]\).
By the chain rule and Proposition~\ref{prop:partial-y},
\begin{equation}
\frac{d}{dt}\bar{T}_j(x;Y(t),W)
=\sum_{\ell=1}^k \big\langle \partial_{y_\ell}\bar{T}_j(x;Y(t),W),\,y'_\ell-y_\ell\big\rangle.
\end{equation}
For each \(t\), the centers satisfy \(\max_\ell\|y_\ell(t)\|\le R_Y\), so
\(\|\partial_{y_\ell}\bar{T}_j(x;Y(t),W)\|\le \frac{2}{\lambda}(R_X+R_Y)\).
Hence,
\begin{equation}
\Big|\tfrac{d}{dt}\bar{T}_j(x;Y(t),W)\Big|
\le \frac{2}{\lambda}(R_X+R_Y)\sum_{\ell=1}^k \|y'_\ell-y_\ell\|
\le \frac{2k}{\lambda}(R_X+R_Y)\,\Delta_Y.
\end{equation}
Then integrate over \(t\in[0,1]\) for both sides of the inequality, which completes the proof.
\end{proof}

\paragraph*{Stability of the empirical M--step.}
Let \(X_{1:n}\subset K\) and define the sufficient statistics
\begin{equation}
\bar{T}_{ij}:=\bar{T}_j(X_i;Y,W),\qquad
S_j:=\sum_{i=1}^n \bar{T}_{ij},\qquad
N_j:=\sum_{i=1}^n \bar{T}_{ij}\,X_i.
\end{equation}
thus the M--step center update is \(y^{\text{new}}_j:=N_j/S_j\), as defined in Def.~\ref{def:M-step}.
Consider perturbing one sample \(X_k\mapsto X_k+h\) with \(\|h\|\le \delta\).
Let \(\Delta T_{kj}:=\bar{T}_j(X_k+h)-\bar{T}_j(X_k)\), \(\Delta S_j:=\Delta T_{kj}\),
and \(R_{j,\max}:=\max_{1\le i\le n}\|X_i-y^{\text{new}}_j\|\).

\begin{theorem}[Per-sample leverage bound]\label{thm:per-sample}
With the notation above, for any \(\delta\ge 0\),
\begin{equation}
\big\|y^{\text{new}}_j(X_k+h)-y^{\text{new}}_j(X_k)\big\|
\ \le\
\frac{\;\delta+\frac{2D_Y}{\lambda}\,\delta\,R_{j,\max}\;}
     {\,S_j-\frac{2D_Y}{\lambda}\,\delta\,}.
\end{equation}
Equivalently, writing \(p_j:=S_j/n\),
\begin{equation}
\big\|y^{\text{new}}_j(X_k+h)-y^{\text{new}}_j(X_k)\big\|
\ \le\
\frac{1}{\,n p_j-\frac{2D_Y}{\lambda}\,\delta\,}\,
\Big(1+\frac{2D_Y}{\lambda}\,R_{j,\max}\Big)\,\delta.
\end{equation}
In particular, if \(\delta\le \frac{\lambda}{4D_Y}\,S_j\) then
\begin{equation}
\big\|y^{\text{new}}_j(X_k+h)-y^{\text{new}}_j(X_k)\big\|
\ \le\
\frac{2}{n p_j}\,\Big(1+\frac{2D_Y}{\lambda}\,R_{j,\max}\Big)\,\delta.
\end{equation}
\end{theorem}
\begin{proof}
Only the \(k\)-th terms change, and
\begin{equation}
N_j' - N_j
=\bar{T}_j(X_k+h)\,(X_k+h)-\bar{T}_j(X_k)\,X_k
=\bar{T}_j(X_k+h)\,h+\Delta T_{kj}\,X_k.
\end{equation}
Since \(y^{\text{new}}_j=N_j/S_j\), one checks
\begin{equation}
y^{\text{new}}_j{}'-y^{\text{new}}_j
=\frac{S_j(N_j'-N_j)-(N_j)\Delta S_j}{S_j(S_j+\Delta S_j)}
=\frac{\bar{T}_j(X_k+h)\,h+\Delta T_{kj}\,(X_k-y^{\text{new}}_j)}{S_j+\Delta S_j}.
\end{equation}
Taking norms, using \(\bar{T}_j\le 1\), Theorem~\ref{thm:lip-x-T} to bound
\(|\Delta T_{kj}|\le (2D_Y/\lambda)\,\delta\), and
\(\|X_k-y^{\text{new}}_j\|\le R_{j,\max}\), gives the first inequality. The
remaining proof follow by substituting \(S_j=n p_j\) and the stated using small‑\(\delta\)
condition.
\end{proof}

\begin{remark}[Hard \(k\)-means lacks global Lipschitz continuity]
Let \(H_j(x;Y)=\mathbf{1}\{j=\arg\min_\ell\|x-y_\ell\|\}\) be the hard assignment.
Across a Voronoi boundary, \(H_j(\cdot;Y)\) jumps by \(\Theta(1)\) under arbitrarily
small perturbations of \(x\) or \(Y\); thus no finite global Lipschitz constant exists.
By contrast, for any fixed \(\lambda>0\), the soft assignments \(\bar{T}^{(\lambda)}\)
are globally Lipschitz in \(x\) (Theorem~\ref{thm:lip-x-T}) and locally Lipschitz in \(Y\)
(Theorem~\ref{thm:lipschitz-centres}) on compact data domains.
\end{remark}

\section{Extended Experimental Setup Details}\label{AP:experiments}

\subsection{Implementation Details}
\paragraph*{Hardware environment.} Experiments run on a workstation with an Intel Core i7-12700K CPU and an NVIDIA RTX 5090 GPU. Unless noted otherwise, we execute GPU-aware routines on the \texttt{cuda} backend and report wall-clock times measured on this system.
\paragraph*{Software environment.} All experiments are implemented on Ubuntu~22.04 with Python~3.11.5, PyTorch~2.7.1+cu128 (compiled against CUDA~12.8), and TorchVision~0.22.1+cu128. Supporting libraries include NumPy~1.26.2, SciPy~1.11.4, Pandas~2.1.1, Matplotlib~3.8.0, Seaborn~0.13.0, and scikit-learn~1.2.2.
\paragraph*{Hyperparameters.} Unless stated otherwise, all methods share a stopping rule of \texttt{max\_iter}$=1000$, a tolerance of \texttt{eps}$=10^{-2}$ on the Frobenius norm between successive centre updates, and a mini-batch size of \texttt{batch\_size}$=1000$. Centres are initialised by sampling from the dataset via the \texttt{random} initialiser. Gradient-based optimisation in WCSL uses the \texttt{Adam} optimiser with learning rate $10^{-1}$ and a \texttt{StepLR} scheduler that multiplies the learning rate by $0.99$ each epoch. Unless specified otherwise, the Sinkhorn regularisation parameter is fixed to $0.01$.
\paragraph*{Evaluation metrics.} We approximate entropic Wasserstein distances using Sinkhorn loss with entropy parameter $0.01$, squared Euclidean costs, numerical tolerance $10^{-6}$ for simulated data and $10^{-5}$ for image data, and up to $1000$ Sinkhorn iterations.
\paragraph*{Repetitions.} Each experiment is run with fixed random seeds for reproducibility, and we report the mean over multiple independent runs to mitigate variance.

\subsection{Experiment Details for Sec.~\ref{sec:distance}}
\paragraph*{Synthetic Gaussian data.} We generate $n=10^4$ points in $d=100$ from a Gaussian distribution whose covariance matrix is assembled from a randomly sampled orthonormal basis.
\paragraph*{MNIST benchmark.} We evaluate on the MNIST handwritten digit benchmark \citep{lecun2002gradient} using its standard train/test split.

\subsection{Experiment Details for Sec.~\ref{sec:shapley}}
\paragraph*{Goal and high-level protocol.}
This experiment evaluates whether a coreset $\nu_k$ preserves \emph{downstream interpretability} when used as the \emph{background distribution} required by SHAP explainers.
We train a predictive model $f$ on the training split, and then compute Shapley/SHAP attributions on the test split using either:
(i) the full training set as background (reference), or
(ii) a coreset of size $k{=}64$ constructed from the same training split (approximation).
We then report attribution reconstruction errors between these two SHAP outputs.

\paragraph*{Datasets, splits, and preprocessing.}
We follow the setup of \citet{jethani2022fastshap} using the \texttt{shap} package datasets:
\textbf{Census} (tabular classification) and \textbf{Diabetes} (tabular regression).
We use an 80/20 train/test split; for classification we use a stratified split, while for regression we use a random split.
All features are standardized using \texttt{StandardScaler} fitted on the training split and applied to both splits.

\paragraph*{Predictive model and SHAP explainer.}
On each dataset, we fit a gradient-boosted tree model (scikit-learn) on the 80\% training split.
Unless otherwise noted, we use $400$ estimators, learning rate $0.05$, and maximum depth $3$.
We compute attributions with \texttt{shap.TreeExplainer} \citep{lundberg2017unified}.
For Census, we explain a scalar output (e.g., the score/probability of the positive class); for Diabetes, we explain the regression output.
All methods share the same trained model; only the \emph{background distribution} differs.

\paragraph*{Shapley values with a background distribution.}
Let $x\in\mathbb{R}^d$ be a test point and let $\nu$ be a background distribution on $\mathbb{R}^d$.
Define the standard ``hybrid'' value function for a subset of features $S\subseteq[d]$ by
\begin{equation}
v(S; x,\nu)
:= \mathbb{E}_{y\sim \nu}\!\Big[f\big(x_S,\, y_{[d]\setminus S}\big)\Big],
\end{equation}
where $(x_S, y_{[d]\setminus S})$ denotes the vector obtained by taking features in $S$ from $x$ and the remaining features from a background draw $y$.
The Shapley value for feature $j$ is then
\begin{equation}
\label{eq:shapley-def-apx}
\phi_j(x;\nu)
=\!\!\sum_{S\subseteq[d]\setminus\{j\}}
\frac{|S|!\,(d-|S|-1)!}{d!}\,
\Big(v(S\cup\{j\};x,\nu)-v(S;x,\nu)\Big),
\end{equation}
which is the unique attribution satisfying efficiency, symmetry, dummy, and additivity \citep{shapley1953value}.
In practice, directly evaluating~\eqref{eq:shapley-def-apx} is combinatorial, and \texttt{TreeExplainer} provides an efficient computation specialized to tree ensembles \citep{lundberg2017unified}.

\paragraph*{Background distributions: full data vs.\ coreset.}
Let the training set be $\{x_i^{\mathrm{tr}}\}_{i=1}^{n_{\mathrm{tr}}}$ and define its empirical measure
\begin{equation}
\mu_{\mathrm{tr}}:=\frac{1}{n_{\mathrm{tr}}}\sum_{i=1}^{n_{\mathrm{tr}}}\delta_{x_i^{\mathrm{tr}}}.
\end{equation}
Each coreset method constructs a discrete weighted measure supported on $k$ atoms:
\begin{equation}
\nu_k := \sum_{m=1}^k \omega_m\,\delta_{y_m},
\qquad \omega_m\ge 0,\ \sum_{m=1}^k \omega_m=1.
\end{equation}
For EMS, both $\{y_m\}$ and $\{\omega_m\}$ are learned.
For K-means and WCSL, the support points are learned but weights are uniform ($\omega_m=1/k$).
For Random, we sample $k$ points from the training set with uniform weights.

With $\nu_k$ as background, the value function becomes a finite weighted average:
\begin{equation}
v(S;x,\nu_k)=\sum_{m=1}^k \omega_m\, f\big(x_S,\, y_{m,[d]\setminus S}\big).
\end{equation}
Consequently, Shapley values under $\nu_k$ decompose into weighted contributions from each atom:
\begin{equation}
\phi_j(x;\nu_k)
= \sum_{m=1}^k \omega_m\, \phi_j(x; y_m),
\end{equation}
where $\phi_j(x;y_m)$ denotes the Shapley value computed with a \emph{singleton} background distribution $\delta_{y_m}$ (equivalently, using $y_m$ to fill missing features in the hybrid evaluation).

\paragraph*{SHAP matrices and reconstruction metrics.}
Let the test set be $\{x_i^{\mathrm{te}}\}_{i=1}^{n_{\mathrm{te}}}$.
We form the SHAP attribution matrices
\begin{equation}
\Phi^{\mathrm{full}} \in \mathbb{R}^{n_{\mathrm{te}}\times d},
\qquad
\Phi^{\mathrm{coreset}} \in \mathbb{R}^{n_{\mathrm{te}}\times d},
\end{equation}
where $\Phi^{\mathrm{full}}_{ij}:=\phi_j(x_i^{\mathrm{te}};\mu_{\mathrm{tr}})$ and
$\Phi^{\mathrm{coreset}}_{ij}:=\phi_j(x_i^{\mathrm{te}};\nu_k)$.
We report three reconstruction errors:
\begin{align}
\mathrm{MeanAE}
&:= \frac{1}{n_{\mathrm{te}}d}\big\|\Phi^{\mathrm{coreset}}-\Phi^{\mathrm{full}}\big\|_1,\\
\mathrm{MaxAE}
&:= \big\|\Phi^{\mathrm{coreset}}-\Phi^{\mathrm{full}}\big\|_\infty,\\
\mathrm{BVE}
&:= \big|b(\nu_k)-b(\mu_{\mathrm{tr}})\big|,
\end{align}
where $\|\cdot\|_\infty$ denotes the elementwise maximum absolute deviation and
$b(\nu)$ is the SHAP \emph{base value} (background expected prediction),
estimated empirically as
\begin{equation}
b(\mu_{\mathrm{tr}})=\frac{1}{n_{\mathrm{tr}}}\sum_{i=1}^{n_{\mathrm{tr}}} f(x_i^{\mathrm{tr}}),
\qquad
b(\nu_k)=\sum_{m=1}^k \omega_m f(y_m).
\end{equation}
This base value is the constant in the SHAP decomposition $f(x)=b(\nu)+\sum_{j=1}^d \phi_j(x;\nu)$, so BVE captures how well the coreset preserves the explainer's reference level.

\paragraph*{Repetitions and randomness control.}
All reported numbers average over five runs with independent random seeds.
Across runs, randomness may arise from data splitting (when applicable) and/or coreset initialization (e.g., random initialization for K-means and EMS). See table \ref{tab:real_world} for the results.

\begin{table}[t]
\centering
\caption{Performance comparison on real-world datasets. Each entry reports distance / runtime (seconds).}
\label{tab:real_world}
\resizebox{\textwidth}{!}{
\begin{tabular}{llccccccc}
\toprule
\multirow{3}{*}{Dataset} &
\multirow{3}{*}{$k$} &
{Random} &
{K-means} &
{WCSL} &
\multicolumn{4}{c}{EMS} \\
\cmidrule(lr){3-9}
& & & &
$\lambda=10^{-2}$ &
$\lambda=0$ &
$\lambda=10^{-3}$ &
$\lambda=10^{-2}$ &
$\lambda=10^{-1}$ \\
& &
\multicolumn{1}{c}{distance / time} &
\multicolumn{1}{c}{distance / time} &
\multicolumn{1}{c}{distance / time} &
\multicolumn{1}{c}{distance / time} &
\multicolumn{1}{c}{distance / time} &
\multicolumn{1}{c}{distance / time} &
\multicolumn{1}{c}{distance / time} \\
\midrule
\multirow{5}{*}{\textbf{Transactions}}
& 20  & 74.47 / -- & 46.47 / 1.43 & \textbf{46.45} / 76.57 & 46.47 / \textbf{1.37} & 46.48 / 1.66 & 46.48 / 1.45 & 46.47 / 1.85 \\
& 50  & 68.85 / -- & 44.88 / \textbf{1.63} & \textbf{44.82} / 148.02 & 44.88 / 2.08 & 44.87 / 1.95 & 44.87 / 2.26 & 44.84 / 2.20 \\
& 100 & 65.65 / -- & 43.67 / \textbf{1.76} & \textbf{43.56} / 203.11 & 43.66 / 2.24 & 43.65 / 2.81 & 43.65 / 3.00 & 43.61 / 3.00 \\
& 200 & 63.13 / -- & 42.48 / 3.15 & \textbf{42.26} / 307.59 & 42.42 / \textbf{2.86} & 42.42 / 3.65 & 42.40 / 6.36 & 42.34 / 6.46 \\
& 500 & 59.54 / -- & 40.94 / 4.01 & \textbf{40.44} / 633.87 & 40.77 / \textbf{3.50} & 40.77 / 4.21 & 40.75 / 8.47 & 40.65 / 16.14 \\
\midrule
\multirow{5}{*}{\textbf{ChemReact100}}
 & 20  & 167.64 / -- & 104.38 / \textbf{0.065} & \textbf{94.80} / 44.82 & 96.71 / 0.079 & 97.41 / 0.081 & 96.93 / 0.108 & 97.25 / 0.104 \\
 & 50  & 149.77 / -- & 110.82 / \textbf{0.111} & \textbf{93.69} / 143.95 & 98.02 / 0.176 & 97.65 / 0.135 & 97.34 / 0.125 & 97.06 / 0.135 \\
 & 100 & 140.06 / -- & 108.89 / \textbf{0.207} & \textbf{94.07} / 224.75 & 97.80 / 0.214 & 97.83 / 0.221 & 97.94 / 0.245 & 97.95 / 0.285 \\
 & 200 & 131.66 / -- & 106.79 / 0.460 & \textbf{93.97} / 364.31 & 98.39 / \textbf{0.400} & 98.32 / 0.452 & 98.38 / 0.434 & 98.45 / 0.776 \\
 & 500 & 124.58 / -- & 104.74 / \textbf{0.522} & \textbf{93.67} / 689.74 & 99.09 / 0.579 & 99.05 / 0.639 & 99.08 / 0.792 & 99.15 / 1.45 \\
\midrule
\multirow{5}{*}{\textbf{Music}}
 & 20  & 467.49 / -- & 357.23 / 0.048 & 219.44 / 20.88 & \textbf{212.53} / 0.042 & 212.58 / \textbf{0.041} & 212.64 / 0.047 & 212.73 / 0.068 \\
 & 50  & 379.75 / -- & 403.59 / 0.081 & 214.71 / 51.60 & \textbf{205.09} / \textbf{0.071} & 205.62 / 0.104 & 205.77 / 0.098 & 205.12 / 0.125 \\
 & 100 & 352.66 / -- & 360.83 / 0.147 & 213.85 / 87.80 & \textbf{197.09} / 0.162 & 197.80 / \textbf{0.122} & 197.32 / 0.126 & 197.42 / 0.302 \\
 & 200 & 333.36 / -- & 405.85 / 0.254 & 205.08 / 155.39 & \textbf{190.86} / 0.251 & 191.42 / \textbf{0.200} & 191.50 / 0.201 & 191.29 / 0.560 \\
 & 500 & 303.75 / -- & 358.80 / 0.362 & 194.03 / 364.59 & \textbf{186.82} / 0.323 & 187.14 / 0.361 & 187.72 / 0.444 & 187.37 / 0.716 \\
\bottomrule
\end{tabular}}
\end{table}

\subsection{Limitations and Future Directions: Large-Scale Image Datasets}
\label{AP:limitations-images}

A limitation of this work is that we do not evaluate EMS on large-scale image datasets or image distillation tasks. This is by design. Our discrepancy is based on squared Euclidean geometry and entropic OT costs, which are well suited to tabular data and controlled synthetic settings, but pixel-wise similarity is generally misaligned with semantic similarity in images. As a result, Wasserstein or Sinkhorn objectives applied directly in pixel space often fail to produce meaningful visual or semantic summaries, even when transport costs are small.

A natural extension is to apply EMS in a learned representation space rather than in pixel space. Let $\psi(\cdot)$ denote a fixed or lightly tuned feature extractor, such as a pretrained vision encoder, and consider the pushforward measure $\mu_\psi := \psi_{\#}\mu$. Running EMS on $\mu_\psi$ yields a compact weighted measure
\[
\nu_k^\psi = \sum_{j=1}^k \omega_j \delta_{z_j},
\]
where the atoms $\{z_j\}$ summarize the dataset geometry in feature space. This connects closely to feature-space barycenter constructions, moment-matching objectives, and recent dataset distillation approaches based on Wasserstein-type criteria \citep{liu2025dataset}.

Two practical issues then arise: (i) how to realize feature-space atoms as images (e.g., via nearest-neighbor selection or by optimizing synthetic images under a perceptual loss with appropriate priors), and (ii) how to evaluate semantic fidelity using downstream performance or representation-level metrics rather than pixel-based distances. Addressing these issues would extend the applicability of EMS to modern vision tasks while preserving its core advantages of fast, stable, EM-style updates.


\end{document}